\documentclass[journal]{IEEEtran}
\usepackage{soul}
\usepackage{url}
\usepackage[small]{caption}
\usepackage{graphicx}
\usepackage{epstopdf}
\usepackage{amsfonts}
\usepackage{amsmath}
\usepackage{amsthm}
\usepackage{booktabs}
\usepackage{xcolor}
\usepackage{algorithmic}
\usepackage[ruled]{algorithm2e}
\usepackage{comment}
\usepackage{subfigure}
\usepackage{bbm}
\usepackage{blindtext}

\usepackage{array}

\newcommand{\RNum}[1]{\uppercase\expandafter{\romannumeral #1\relax}}
\newtheorem{Def}{Definition}
\newtheorem{Theo}{Theorem}

\newtheorem{lem}{Lemma}

\ifCLASSINFOpdf
\else
\fi

\begin{document}
%
\title{Fairness Constraints in Semi-supervised Learning}
%
%
%

\author{Tao Zhang,
	Tianqing Zhu,  Mengde Han, Jing Li, Wanlei Zhou,~\IEEEmembership{Senior Member, IEEE} and Philip S. Yu~\IEEEmembership{Fellow, IEEE}

\thanks{ 
	Tianqing Zhu$ ^{*} $ is the corresponding author.
	
	Tao Zhang, Tianqing Zhu,  Mengde Han, Wanlei Zhou are with the Centre for Cyber Security and Privacy, School of Computer Science, University of Technology Sydney, Sydney
	Australia. Email: $ \{ $Tao.Zhang-3@student.uts.edu.au,  Tianqing.Zhu@uts.edu.au, Mengde.Han@student.uts.edu.au, Wanlei.Zhou@uts.edu.au$ \} $
	
	Jing Li is in the centre for Artificial Intelligence, University of Technology Sydney. Email: $ \{ $jing.li-20@student.uts.edu.au$ \} $
	
	Philip S, Yu is with the Department of Computer Science University of Illinois at Chicago Chicago, Illinois, USA. 
	Email: $ \{ $psyu@uic.edu$ \} $

}}
\IEEEtitleabstractindextext{%


\begin{abstract}
	
	Fairness in machine learning has received considerable attention. However, most studies on fair learning focus on either supervised learning or unsupervised learning. Very few consider semi-supervised settings. Yet, in reality, most machine learning tasks rely on large datasets that contain both labeled and unlabeled data. 
	One of key issues with fair learning is the balance between fairness and accuracy. Previous studies arguing that increasing the size of the training set can have a better trade-off. We believe that increasing the training set with unlabeled data may achieve the similar result.
	Hence, we develop a framework for fair semi-supervised learning, which is formulated as an optimization problem. This includes classifier loss to optimize accuracy, label propagation loss to optimize unlabled data prediction, and fairness constraints over labeled and unlabeled data to optimize the fairness level.  The framework is conducted in logistic regression and support vector machines under the fairness metrics of disparate impact and disparate mistreatment.
	We theoretically analyze the source of discrimination in semi-supervised learning via bias, variance and noise decomposition. Extensive experiments show that our method is able to achieve fair semi-supervised learning, and reach a better trade-off between accuracy and fairness than fair supervised learning.


\end{abstract}
}

\maketitle
\IEEEdisplaynontitleabstractindextext

%
\IEEEpeerreviewmaketitle

%
{\section{Introduction}}
Machine learning algorithms, as useful decision-making tools, are widely used in the society. These algorithms are often assumed to be paragons of objectivity. However, many studies show that the decisions made by these models can be biased against certain groups of people.
For example, Chouldechova et al. \cite{Chouldechova2017} found evidence of racial bias in recidivism prediction where   Black defendants are particularly likely to falsely be flagged as future criminals, and Obermeyer et al. \cite{Obermeyer447} found prejudice in health care systems where Black patients assigned the same level of risk by the algorithm are sicker than White patients. These events prove that discrimination can arise from machine learning algorithms, and do harm to the fundamental rights of human beings. 
Given the widespread use of machine learning to support decisions over loan allocations, insurance coverage, the best candidate for a job, and many other basic precursors to equity, fairness in machine learning has become a significantly important issue.  Thus, designing machine learning algorithms that treat all groups equally is critical.

Over the past few years, many fairness metrics have been proposed to define what is fairness in machine learning. Popular fairness metrics include statistical fairness \cite{Chouldechova2017,Zafar2017,Hardt2016}, individual fairness \cite{Dwork2012,jung2019eliciting,8885331,dwork2020individual} and casual fairness \cite{kusner2017counterfactual,kilbertus2017avoiding,wu2019pc}. Meanwhile, a great many algorithms have been developed to address fairness issues for both supervised learning settings  \cite{Zafar2017,Hardt2016,Dwork2012,Zemel2013,Song2019} and unsupervised settings \cite{chierichetti2017fair,Kleindessner2019,backurs2019scalable,rsner_et_al:LIPIcs:2018:9100,Chen2019ProportionallyFC}.
Generally, these studies have focused on two key issues: how to formalize the concept of fairness in the context of machine learning tasks, and how to design efficient algorithms that strike a desirable trade-off between accuracy and fairness. What is lacking is study that consider semi-supervised learning (SSL) scenarios.

In real-world machine learning tasks, the data used for training is often a combination of labeled and unlabeled data. Therefore, fair SSL is a vital area of development. Like the other learning settings, achieving a balance between accuracy and fairness is also a key issue. 
However, given the mix of labeled and unlabeled data in SSL, the trade-off solutions developed for supervised and unsupervised learning do not directly apply. 
According to \cite{Chen2018}, increasing the size of the training set can create a better trade-off.  This finding sparked an idea over whether the trade-off might be improved via unlabeled data. Unlabeled data is abundant and, if it could be used as training data, we may be able to avoid the need to make the compromise between fairness and accuracy.
Generally, fair SSL has two challenges: 1) how to achieve fair learning from both labeled and unlabeled data; 2) how to make use of unlabeled data to achieve a better trade-off  between accuracy and fairness.



To solve these challenges, we propose a framework of fair SSL that can support multiple classifiers and fairness metrics.  
The framework is formulated as an optimization problem, where the objective function includes a loss for both the classifier and label propagation, and fairness constraints over labeled and unlabeled data. Classifier loss is to optimize the accuracy of training result; label propagation loss is to optimize the label predictions on unlabeled data; the fairness constraint is to adjust the fairness level as desirable.
The optimization includes two steps. In the first step, fairness constraints enforce weights update towards a fair direction. This step can be solved by a convex problem and convex-concave programming when disparate impact and disparate mistreatment are used as fairness metrics respectively. In the second step, 
updated weights further direct labels assigned to unlabeled data in a fair direction by label propagation. Labels for unlabeled data can be calculated in a closed form. In this way, labeled and unlabeled data are used to achieve a better trade-off between accuracy and fairness.
With this strategy, we can control the level of discrimination in the model and, therefore, provide a machine learning framework that offers fair SSL.
Our approach incorporates a wide range of fairness definitions such as disparate impact and disparate mistreatment, which is guaranteed to yield an accurate fair classifier.
%

With the aim of achieving fair SSL, the contributions of this paper are three-fold. 
\begin{itemize}
	\item First, we propose a framework that is able to achieve fair SSL that supports multiple classifiers and fairness metrics of disparate impact and disparate mistreatment.
	This framework enables the use of unlabeled data to achieve a better trade-off between fairness and accuracy.
\end{itemize}
\begin{itemize}
	\item Second, we consider different cases of fairness constraints on labeled and unlabeled data, and analyze the impacts of these constraints on the training results. This helps us how to control the fairness level in practice.
\end{itemize}
\begin{itemize}
	\item Third, we theoretically analyze the sources of discrimination in SSL, and conduct extensive experiments to validate the effectiveness of our proposed method.
\end{itemize}
The rest of this paper is organized as follows. The preliminaries is given in Section \RNum{2}, and the proposed framework is given in Section \RNum{3}. Section \RNum{4} presents the discrimination analysis, and the experiments are set out in Section \RNum{5}. The related work appears in Section \RNum{6}, with the conclusion in Section \RNum{7}.


\section{Preliminaries}
\subsection{Notations and Problem Definition}
Let $ X=\{x_{1},...,x_{K}\}^{T}  \in \mathbb{R}^{K \times v}$ denote the training data matrix, where $ K $ is the number of data point and $ v $ is the number of unprotected attributes; $ \boldsymbol{z}= \{z_{1},...,z_{n}\} \in \{0,1\}$  denotes the protected attribute, e.g., gender or race.
Labeled dataset is denoted as $ \mathcal{D}_{l} = \{x_{i},z_{i},y_{l,i}\}_{i=1}^{K_{l}}$ with $ K_{l} $ data points, and  $ \boldsymbol{y_{l}}= \{y_{l,1},...,y_{l,K_{l}}\}^{T}\in \{0,1\} $ is the label for the labeled dataset. Unlabeled dataset is denoted as $ \mathcal{D}_{u} = \{x_{i},z_{i}\}_{i=1}^{K_{u}}$ with $ K_{u} $ data points, and $ \boldsymbol{y_{u}}= \{y_{u,1},...,y_{u,K_{u}}\}^{T}\in \{0,1\} $ is the predicted labeled for the unlabeled dataset.
Our objective is to learn a classification model $ f(\cdot) $ with the parameters $ \boldsymbol{w} $ and $\boldsymbol{y_u}$ over discriminatory datasets $ \mathcal{D}_{l}  $ and $ \mathcal{D}_{u} $ that delivers high accuracy with low discrimination. 
\subsection{Fairness Metric}
Fairness is usually assessed on the basis of protected/unprotected groups of individuals (defined by their sensitive attributes). In our framework, we have applied disparate impact and disparate mistreatment as the fairness metrics \cite{Zafar2017,Zafar2017a}.

\subsubsection{Disparate impact}
A classification model does not suffer disparate impact if,
\begin{equation}
Pr(\hat{y}=1|z=1)= Pr(\hat{y}=1|z=0)
\end{equation}
where $\hat{y}$ is the predicted label. 
The probability that a classifier predicts positive class for a data point regardless of the data point is in the protected group or unprotected group. This means that positive prediction is the same for both values of the sensitive feature $ z $, then there is no disparate impact.

\subsubsection{Disparate mistreatment}
A binary classifier will not suffer different mistreatment if the misclassification rate of
different groups with different values of sensitive feature $ z $ is the same. 
Specifically, the misclassification rate can be measured as fractions over the group class and the ground truth label, that is, as the false positive rate and false negative rate. Here, three different kind of disparate mistreatments are adopted to evaluate the discrimination as follows,

\noindent overall misclassification rate (OMR):
\begin{equation}
Pr(\hat{y} \neq y|z=1)= Pr(\hat{y} \neq y|z=0)
\end{equation}
false positive rate (FPR):
\begin{equation}
Pr(\hat{y} \neq y|z=1,y=0)= Pr(\hat{y} \neq y|z=0,z=0)
\end{equation}
false negative rate (FNR):
\begin{equation}
Pr(\hat{y} \neq y|z=1,y=1)= Pr(\hat{y} \neq y|z=0,y=1)
\end{equation}

In most cases, a classifier suffers discrimination in terms of disparate impact or disparate mistreatment. The discrimination level is defined as the differences in rates between different groups.
\begin{Def}
	Let $ \gamma_{z} $ denote the extent of discrimination in group $ z $ in terms of a fairness metric.
	The discrimination level $\Gamma(\hat{y})$ is measured by the difference between $ \gamma_{z} $, denoted as 
	$\Gamma(\hat{y}) = |\gamma_{0}(\hat{y})-\gamma_{1}(\hat{y})|$.
\end{Def}

\subsection{Bias, variance and noise}
According to \cite{Chen2018}, bias, variance and noise decomposition can be used to analyze  the sources of discrimination.
In the following, we give definition of main prediction, bias, variance and noise. 
The main prediction is defined as  $y_{m}(x,z)=  \mathop{\arg\min}_{\hat{y}} \mathbbm{E}_{D}[L(y,\hat{y})]$ for a loss function $ L $ and a set of training sets $ 
\mathcal{D} $, where $ y $ is the true label; $ \hat{y} $ is the predicted label, and the expectation
is taken over the training sets in $ \mathcal{D} $.
\begin{Def}
	(Bias, variance and noise) Following \cite{Domingos2000}, bias $ B $, variance $ V $ and noise $ N $ at a point $ (x,z) $ are defined as,
	\begin{align}
	&B(\hat{y},x,z)=L({y}^{*}(x,z),y_{m}(x,z))\\
	&V(\hat{y},x,z)= \mathbbm{E}_{D}[L(y_{m}(x,z),\hat{y}(x,z)]\\
	&N(\hat{y},x,z)=\mathbbm{E}_{D}[L({y}^{*}(x,z),y(x,z))]
	\end{align}	
\end{Def}	
\noindent	where ${y}^{*}$ is the optimal prediction that attains the smallest expected error.
Bias is the loss between the main prediction and the optimal prediction. Variance is the average loss induced by predictions relative to the main prediction.
Noise is an inevitable part of the loss, which is irrelevant of the learning model.

The bias-variance-noise decomposition is suitable for discrimination analysis because the loss is related to the misclassification rate.
For instance, when zero-one loss function is applied, misclassification rate can be presented as,
\begin{equation}
\mathbbm{E}_{D}[L(y,\hat{y})]=\mathbbm{E}_{D}[ \hat y \neq y|z=0] + \mathbbm{E}_{D}[\hat y \neq y|z=1]\\ \notag 
\end{equation}
Loss function can be decomposited into false positive rate and false negative rate. When these rates are known, true positive rate and true negative rate can also be calculated.  
This means that many fairness metrics, such as demographic parity and equal opportunity, might be explained via bias, variance and noise decomposition.

\section{Proposed Method}
In this section, we first present the proposed framework in section 3.1. Then fairness metrics of disparate impact and disparate mistreatment  in logistic regression are analyzed in section 3.2.  Section 3.3 shows the case in support vector machines, and finally a discussion is given in section 3.4.

\subsection{The Proposed Framework}
We formulate the framework of fair SSL as following, including the classification loss, the label propagation loss and fairness constraints.
\begin{align}
&\min_{\boldsymbol{w},\boldsymbol{y_{u}}} L_{1}(\boldsymbol{w},\boldsymbol{y_{u}}) +\alpha L_{2}(\boldsymbol{y_{u}})
&s.t. \hspace*{0.7em}  s(\boldsymbol{w})\leq c 
\end{align}
where $ L_1 $ is the classifier loss between predicted labels and true labels; $ L_{2} $ is the loss of label propagation from labeled data to unlabeled data; $ \alpha $ is a parameter to balance the loss;
$ s(\boldsymbol{w}) $ is the fairness constraints and $ c $ is a threshold.

\subsubsection{Classifier Loss}
A classifier loss function evaluates how well a specific algorithm models the given dataset. When different algorithms are used to train datasets, such as logistic regression or neural networks, a corresponding loss function is applied to evaluate the accuracy of the model.
\subsubsection{Label Propagation Loss}
Label propagation is performed by an efficient SSL algorithm that allocates labels to unlabeled data \cite{zhu2002learning}. In our framework, label propagation is based on a fully- connected graph where the nodes represent data points - labeled and unlabeled data points.
The edges in the graph between each pair of data points $ i $ and $ j $ is weighted. The closer the two points are in Euclidean space $ d_{ij} $, the greater the weight $ \theta_{ij} $. We chose a Gaussian similarity function to calculate the weights, given as follows:
\begin{equation}
\theta_{i j}=\exp \left(-\frac{d_{i j}^{2}}{\sigma^{2}}\right)=\exp \left(\frac{\sum_{d}\left(x_{i}^{d}-x_{j}^{d}\right)^{2}}{\sigma^{2}}\right)
\end{equation}
where $ \sigma $ is a length scale parameter. This parameter has an impact on the graph structure; hence, the value of $ \sigma $ needs to be selected carefully \cite{wang2007label}.

Given the whole dataset, an adjacent matrix is constructed as $ \Theta={\theta_{ij} \in \mathbb{R}^{K \times K}},\forall i,j \in 1,...,K $ ($ K=K_{l}+K_{u} $), and the degree matrix $ {D} $ is constructed as a diagonal matrix whose $i$-th diagonal element is  $ d_{ii} = \sum_{j=1}^{K} \theta_{ij} $. When a binary classification is considered,
the vector $ \boldsymbol{y}=[\boldsymbol{y}_{l};\boldsymbol{y}_{u}] \in \mathbb{R}^{K} $ is a cluster indicator vector that includes labels of labeled and unlabeled data. 
Therefore, the label propagation loss for $ L_{2} $ through SSL can be expressed as,
\begin{equation}
L_{2}= \min_{\forall f_{i}=y_{{u,i}, i=1,...,K_u}} \operatorname{Tr}( \boldsymbol{y}^{T} U \boldsymbol{y} )
\end{equation}
where $ Tr $ denotes the trace, and $  {U} $ is Laplacian matrix, calculated as $ {U}={D}-\Theta$.

\subsubsection{Fairness Constraints} 
Adding fairness constraints is a useful method to enforce fair learning with in-processing methods.
In SSL, labeled data and unlabeled data have different impacts on discrimination. One reason is that predicting labels for unlabeled data will bring noise to the labels. 
Another reason is that labeled data and unlabeled data may have different data distributions. 
Therefore, the discrimination inherently in unlabeled data is different from the discrimination in labeled data.
For these reasons, we impose fairness constraints on labeled and unlabeled data to measure discrimination, respectively.
We consider four cases of fairness constraints enforced on the training data:
\begin{itemize}
	\item  1. Labeled data: $ s_{1}(\boldsymbol{w}) \leq c $.
	\item  2. Unlabeled data: $ s_{2}(\boldsymbol{w}) \leq c $.
	\item  3. Combined labeled and unlabeled data: $ s_{1}(\boldsymbol{w}) \leq c_1 $ (for labeled data) and $ s_{2}(\boldsymbol{w}) \leq c_2 $ (for unlabeled data).
	\item  4. Mixed labeled and unlabeled data:  $ s(\boldsymbol{w}) \leq c $.
\end{itemize}  
where $ c$ is a discrimination threshold, that is adjusted to control the trade-off between accuracy and fairness. 
Note that many fairness constraints \cite{Zafar2017,Zafar2017a,agarwal2018reductions} have been proposed to enforce various fairness metrics, such as disparate impact and disparate mistreatment, and these fairness constraints can be used in our framework. The basic idea to design fairness constraints is that using the covariance between the users’ sensitive attributes and the signed distance between the feature vectors restricts the correlation between sensitive attributes and the classification result. 
This can be described as,
\begin{equation}
|\frac{1}{K}  g_{w} \left(\boldsymbol{z}-\bar{z}\right) |\leq c 
\end{equation}
where $ g_{w} $ is the signed distance between the feature vectors and the decision boundary of a classifier. The form of  $ g_{w} $ is different in fairness metrics, and we list them in the following,
\begin{itemize}
	\item (Disparate impact)
	\begin{equation}
	g_{w}=\boldsymbol{w}^{T}  X 
	\end{equation}
	\item (Overall misclassification rate)
	\begin{equation}
	g_{w}=\min \left(0, \boldsymbol{y} \boldsymbol{w}^{T}  X \right) 
	\end{equation}
	\item (False positive rate)
	\begin{equation}
	g_{w}=\min \left(0, \frac{\boldsymbol{1}-\boldsymbol{y}^{T}}{2} \boldsymbol{y} \boldsymbol{w}^{T}  X \right)
	\end{equation}
	\item (False negative rate)
	\begin{equation}
	g_{w}=\min \left(0, \frac{\boldsymbol{1}+\boldsymbol{y}^{T}}{2} \boldsymbol{y} \boldsymbol{w}^{T}  X  \right)
	\end{equation}
\end{itemize}
Note that labels appear in the disparate mistreatment, and do not appear in the disparate impact. This could result in differences when the four cases of fairness constraint  are used. These fairness metrics will be analyzed in the following sections.

\subsection{A Case of Logistic Regression}
In this section, we focus on Eq. (8), which is the case of a binary logistic regression (LR) classifiers. The classifier is subjected to the fairness metric of demographic parity with mixed labeled and unlabeled data. The objective function of LR is defined as,
\begin{equation}
L_{1}^{LR}= -\ln (\boldsymbol{p}) \boldsymbol{y}- \ln (\boldsymbol{1}-\boldsymbol{p}) (\boldsymbol{1}-\boldsymbol{y}) 
\end{equation}
where  $\boldsymbol{p}=\frac{1}{1+e^{-\boldsymbol{w}^{T} X}}$ is the probability distribution of mapping $ X $ to the class label $ \boldsymbol{y} $;  $ \boldsymbol{1} $ denotes a column vector with all its elements being 1. 
Given the logistic regression loss, the label propagation loss and the fairness metric, the optimized problem (8) adopts the form, 
\begin{equation}
\begin{aligned}
&\min_{\boldsymbol{w}, \boldsymbol{y_u}} -\ln (\boldsymbol{p}) \boldsymbol{y}- \ln (\boldsymbol{1}-\boldsymbol{p}) (\boldsymbol{1}-\boldsymbol{y})  + \alpha{Tr}( {\boldsymbol{y}} ^{T} U {\boldsymbol{y}} ) \\ 
&s.t.  \hspace*{1em}|\frac{1}{K}  g_{w} \left(\boldsymbol{z}-\bar{z}\right) |\leq c \\ 
\end{aligned}
\end{equation}

\subsubsection{Disparate impact}
The fairness constraints used here is from \cite{Zafar2017}, which is defined as the covariance between the sensitive attribute and the signed distance from feature vectors to the decision boundary. Therefore, $c$ is the threshold of covariance that controls discrimination level. A smaller $ c $ indicates a lower discrimination level.
The optimization of problem (17) includes two parts: learning the weights $ \boldsymbol{w} $ and predicted labels of unlabeled data $ \boldsymbol{y_{u}} $. The problem is solved by updating $ \boldsymbol{w} $ and $ \boldsymbol{y_{u}} $ iteratively as follows. 

\textbf{Solving $ \boldsymbol{w} $ when $ \boldsymbol{y_{u}} $ is fixed,} the problem (17) becomes
\begin{equation}
\begin{aligned}
&\min_{\boldsymbol{w}} -\ln (\boldsymbol{p}) \boldsymbol{y}- \ln (\boldsymbol{1}-\boldsymbol{p}) (\boldsymbol{1}-\boldsymbol{y}) \\  
&s.t. \hspace*{1em} |\frac{1}{K}  \boldsymbol{w}^{T} X \left(\boldsymbol{z}-\bar{z}\right) |\leq c \\ 
\end{aligned}
\end{equation}
Note that problem (18) is a convex problem that can be written as a regularized optimization problem by moving fairness constraints to the objective function. The optimal $ \boldsymbol{w}^* $ can then be calculated by using KKT conditions.

\textbf{Solving $ \boldsymbol{y_{u}} $ when $ \boldsymbol{w} $ is fixed,} the problem (17) becomes
\begin{equation}
\min_{\boldsymbol {y_{u}}} -\ln (\boldsymbol{p}) \boldsymbol{y}- \ln (\boldsymbol{1}-\boldsymbol{p}) (\boldsymbol{1}-\boldsymbol{y})  + \alpha{Tr}( \boldsymbol{y} ^{T} U \boldsymbol{y} ) 
\end{equation}
Given that problem (19) is also a convex problem, the optimal $ \boldsymbol{y_u} $ can be obtained from the deviation of  $\boldsymbol{y_{u}} $ in problem (19). In order to calculate  $\boldsymbol{y_{u}} $ conveniently, we split Laplacian matrix $ L $ into four blocks after the $ l $-th row and the $ l $-th column: 
$U=\left[\begin{array}{cc}{U_{l l}} & {U_{l u}} \\ {U_{u l}} & {U_{u u}}\end{array}\right]$. 
The deviation of Eq.(19) is then calculated w.r.t. $ \boldsymbol{y_{u}} $ and setting to zero, we have 
\begin{equation}
\alpha(2\boldsymbol{y_{u}}U_{uu} + U_{ul}\boldsymbol{y_{l}} + (\boldsymbol{y_{l}}U_{lu})^{T}) - [(ln(\boldsymbol{p}))^{T}+(ln(\boldsymbol{1}-\boldsymbol{p}))^{T}]=0
\end{equation} 
Note that $ U $ is a symmetric matrix and, after simplification, the closed updated form of $ \boldsymbol{y_{u}} $ can be derived from
\begin{equation}
\boldsymbol{y_{u}} = -U_{uu}^{-1} (U_{ul}\boldsymbol{y_{l}}+\frac{1}{2\alpha}[(ln(\boldsymbol{p}))^{T}+(ln(\boldsymbol{1}-\boldsymbol{p}))^{T}])  
\end{equation} 
When updating $\boldsymbol{y_{u}} $ with Eq. (21), the value of $\boldsymbol{y_{u}} $ may fall outside range of 0 to 1. Therefore, before using $ \boldsymbol{y_u}$ to update the next $ \boldsymbol{w} $, the value of $ y_{u,i} \in \boldsymbol{y_u},i=1,...,K_u $ is set to,
\begin{equation}
y_{u,i}=\left\{
\begin{aligned}
&1,&y_{u,i} \geq T \\
&0,&y_{u,i} < T
\end{aligned} 
\right.
\end{equation}
where $ T $ is the threshold that determines the classification result.
Then, the optimization problem (16) can be solved by optimizing $ \boldsymbol{w} $ and $ \boldsymbol{y_{u}} $ iteratively. $ \textbf{Algorithm 1} $ summarizes the solution of optimized problem (17) with the disparate impact.

\begin{algorithm}[h]
	\caption{The algorithm of optimizing problem (17) with disparate impact }
	\label{alg:algorithm}
	\textbf{Input}: Labeled dataset $\mathcal{D}_{l}$, unlabeled dataset $\mathcal{D}_{u}$, fairness thresholds $ c $\\
	\textbf{Parameter}: $ T $, $ \sigma $\\
	\textbf{Initialize}: Given random initial values of  $ \boldsymbol{y_{u}} $ \\
	\textbf{Output}: $ \boldsymbol{w} $ and $ \boldsymbol{y_{u}} $
	
	\begin{algorithmic}[1] 
		\STATE Calculate the adjacent matrix $ \Theta $ according to Eq.(9)
		
		
		\REPEAT
		
		\STATE Fix  $ \boldsymbol{y_{u}}$ and update $\boldsymbol{w}  $ with KKT
		\STATE Fix $ \boldsymbol{w} $ and update $ \boldsymbol{y_{u}}$  by Eq.(21) 
		\STATE Set $y_{u,i} \in \boldsymbol{y_{u}}  $ to  0 or 1 by Eq. (22)
		\UNTIL {The optimization problem (17) convergs}
		
		
	\end{algorithmic}
\end{algorithm} 
\subsubsection{Disparate mistreatment}  
Disparate mistreatment metrics include overall misclassification rate, false positive rate and false negative rate.  For simplicity, overall misclassification rate is used to analyze disparate mistreatment. However, false positive rate and false negative rate can also be analyzed easily, and the result of three disparate mistreatment metrics are presented in the experiment. 

With the  overall misclassification rate as the fairness metric, the objective function is denoted as, 
\begin{equation}
\begin{aligned}
&\min_{\boldsymbol{w}} -\ln (\boldsymbol{p}) \boldsymbol{y}- \ln (\boldsymbol{1}-\boldsymbol{p}) (\boldsymbol{1}-\boldsymbol{y}) + \alpha{Tr}( {\boldsymbol{y}} ^{T} U {\boldsymbol{y}} ) \\  
&s.t. \hspace*{1em} |\frac{1}{K}   g_{w} (\mathbf{x}) \left(\boldsymbol{z}-\bar{z}\right) |\leq c \\ 
\end{aligned}
\end{equation}
Note that fairness constraints of disparate mistreatment are non-convex, and the solution of the optimization problem (23) is more challenging than the optimization problem in (17). Next, we convert these constraints into a Disciplined Convex-Concave Program (DCCP). Thus, the optimization problem (23) can be solved efficiently with the recent advances in convex-concave programming \cite{shen2016disciplined}.

The fairness constraint of disparate mistreatment can be split into two terms,
\begin{equation}
\frac{1}{K}|\sum_{ \mathcal{D}_{0}}(0-\bar{z}) g_{w}+\sum_{ \mathcal{D}_{1}}(1-\bar{z}) g_{w}| \leq c
\end{equation}
where $ \mathcal{D}_0 $ and $ \mathcal{D}_1 $ are the subsets of the labeled dataset $ \mathcal{D}_{l} $ and unlabeled dataset $ \mathcal{D}_{u} $ with values $ z = 0 $ and $ z = 1 $, respectively. $ K_0 $ and $ K_1 $ are defined as the number of data points in the $ \mathcal{D}_0 $ and $ \mathcal{D}_1 $, and thus $ \bar{z} $ can be rewriten as $ \bar{z}=\frac{0*K_{0}+1*K_{1}} {K}=\frac{K_{1}}{K}$.
Then the fairness constraint of disparate mistreatment can be rewriten as,
\begin{equation}
\frac{K_1}{K}|\sum_{\mathcal{D}_{0}} g_{w}+\sum_{ \mathcal{D}_{1}} g_{w}| \leq c 
\end{equation}

\textbf{Solving $ \boldsymbol{w} $ when $ \boldsymbol{y_{u}} $ is fixed,} the problem (23) becomes
\begin{equation}
\begin{aligned}
&\min_{\boldsymbol{w}} -\ln (\boldsymbol{p}) \boldsymbol{y}- \ln (\boldsymbol{1}-\boldsymbol{p}) (\boldsymbol{1}-\boldsymbol{y})\\
&s.t. \hspace*{1em} \frac{K_1}{K}|\sum_{ \mathcal{D}_{0}} g_{w}+\sum_{ \mathcal{D}_{1}} g_{w}| \leq c \\ 
\end{aligned}
\end{equation}
The optimization problem (26) is a Disciplined Convex-Concave Program (DCCP) for any convex loss, and can be solved with some efficient heuristics \cite{shen2016disciplined}.

\textbf{Solving $ \boldsymbol{y_{u}} $ when $ \boldsymbol{w} $ is fixed,} the problem (23) becomes
\begin{equation}
\min_{\boldsymbol{y_{u}}} -\ln (\boldsymbol{p}) \boldsymbol{y}- \ln (\boldsymbol{1}-\boldsymbol{p}) (\boldsymbol{1}-\boldsymbol{y})+\alpha{Tr}( {\boldsymbol{y}} ^{T} U {\boldsymbol{y}} ) 
\end{equation}
The solution of Eq. (27) is the same as the solution of the Eq. (21). The closed form of $ \boldsymbol{y_u} $ can be obtained via Eq. (22), and then the optimization problem (23) can be solved by  updating $ \boldsymbol{y_u} $ and $ \boldsymbol{w} $ iteratively.  $ \textbf{Algorithm 2} $ summarizes this process. 
\begin{algorithm}[h]
	\caption{The algorithm of optimizing problem (23) }
	\label{alg:algorithm}
	\textbf{Input}: Labeled dataset $\mathcal{D}_{l}$, unlabeled dataset $\mathcal{D}_{u}$, fairness thresholds $ c $\\
	\textbf{Parameter}: $ T $, $ \sigma $\\
	\textbf{Initialize}: Given random initial values of  $ \boldsymbol{y_{u}} $ \\
	\textbf{Output}: $ \boldsymbol{w} $ and $ \boldsymbol{y_{u}} $
	
	\begin{algorithmic}[1] 
		\STATE Calculate the adjacent matrix $ \Theta $ according to Eq.(9)
		
		\STATE Choose a metric in disparate mistreatment
		\REPEAT
		\STATE Divide $ \mathcal{D} $ into $ \mathcal{D}_{0} $ and $ \mathcal{D}_{1} $ 
		\STATE Calculate $ K_{0}  $ and $ K_{1} $
		\STATE Fix  $ \boldsymbol{y_{u}}$ and update $\boldsymbol{w}  $ with DCCP
		\STATE Fix $ \boldsymbol{w} $ and update $ \boldsymbol{y_{u}}$  by Eq.(21) 
		\STATE Set $y_{u,i} \in \boldsymbol{y_{u}}  $ to  0 or 1 by Eq. (22)
		\UNTIL {The optimization problem (23) convergs}
		
		
	\end{algorithmic}
\end{algorithm}

\subsection{A Case of SVM}
SVM can also be applied in our framework.
SVM uses a  hyperplane $ \boldsymbol{w}^{T}X = 0 $ to classify data points. 
The loss function of SVM is defined as,
\begin{equation}
L_{1}^{SVM}=\frac{1}{K} \left(1-\boldsymbol{y}\left(\boldsymbol{w}^{T} X\right)\right)
\end{equation}
Based on SVM loss, the label propagation and the fairness metrics, the objective function is given as,
\begin{equation}
\begin{aligned}
&\min_{\boldsymbol{w}, \boldsymbol{y_u}} \frac{1}{K} \left(1-\boldsymbol{y}\left(\boldsymbol{w}^{T} X\right)\right)  + \alpha{Tr}( {\boldsymbol{y}} ^{T} U {\boldsymbol{y}} ) \\ 
&s.t.  \hspace*{1em}|\frac{1}{K}  g_{w} \left(\boldsymbol{z}-\bar{z}\right) |\leq c \\ 
\end{aligned}
\end{equation}
Disparate impact and disparate mistreatment can be used in the fairness constraint. Since SVM loss is convex, the solution of problem (28) is similar to the LR case. For simplicity, we omit the process, and show the results in the experiment.

\subsection{Discussion}
Based on above analysis, some conclusions can be drawn:

1. Since unlabeled data do not contain any label information, they do not label biased information so that we can take advantage of the unlabeled data to improve the trade off between accuracy and fairness. In our framework, due to the fairness constraint, the weight  $ \boldsymbol{w} $  is updated towards a fair direction. 
Using the updated $ \boldsymbol{w} $ to update $ \boldsymbol{y_u}$ also ensures that $ \boldsymbol{y_u}$ is directed towards fairness. In this way, fairness is enforced in labeled and unlabeled data by updating $ \boldsymbol{w} $ and $ \boldsymbol{y_u} $ iteratively. Therefore, labels of unlabeled data are calculated in a fair way, which is beneficial to the accuracy of the classifier as well as the fairness of the classifier.

2. Fairness constraints on disparate impact and disparate mistreatment adjust the covariance between the sensitive attribute and the signed distance between feature vectors to the decision boundaries. Disparate impact can be converted into a convex constraint, and disparate mistreatment can be converted into a non-convex constraint.


3. The computed optimal $ \boldsymbol{y_u} $ is decimals, and it cannot be used to update $ \boldsymbol{w} $ directly because only integers are allowed to optimize $ L_1 $ in the next update. Due to this, we need to convert $ \boldsymbol{y_u} $ from decimals to integers to update $ \boldsymbol{w} $. 

\section{Discrimination Analysis}
Following \cite{Chen2018}, we analyze the fairness of the predictive model on the basis of model bias, variance, and noise.
First, we redefine discrimination level to account for randomness in the sampled fair datasets.
\begin{Def}
	The expected discrimination level $ \Gamma(\hat{y}) $  of the predicted model from a random training dataset $ D $ is defined as,
	\begin{equation}
	\Gamma(\hat{y}) = |\mathbbm{E}_{D}[{\gamma_{0}}(\hat{y}_{D})-{\gamma_{1}}(\hat{y}_{D})]|=|\bar{\gamma_{0}}(\hat{y}_{D})-\bar{\gamma_{1}}\hat{y}_{D}) |
	\end{equation}
\end{Def}
Discrimination can be decoupled as discrimination in bias, discrimination in variance and discrimination in noise.
\begin{lem}
	The discrimination with regard to group $ z \in Z $ is defined as, 
	\begin{equation}
	\bar{\gamma_{z}}(\hat{y}) = \bar{B_z} (\hat{y}) + \bar{V_{z}}(\hat{y}) +  \bar{N_{z}}
	\end{equation}
\end{lem}
When the protected group and unprotected group is given, the discrimination level is calculated as,
\begin{equation}
\bar{\Gamma} = |(\bar{B_{0}}-\bar{B_{1}}) + (\bar{V_{0}}-\bar{V_{1}}) + (\bar{N_{0}}-\bar{N_{1}}) |
\end{equation}
and each of the component of Eq. (17) are calculated as,
\begin{align}
&\bar{B_z} (\hat{y}) =\mathbbm{E}_{D} [B(y_{m},x,z)|Z=z] \\
&\bar{V_z} (\hat{y}) =\mathbbm{E}_{D} [c_{v}(x,z)V(y_{m},x,z)|Z=z] \\
&\bar{N_z} = \mathbbm{E}_{D} [c_{n}(x,z) L(y^{*}(x,z), y)|Z=z] 
\end{align}
where $ c_v(x,z) $ and $ c_{n}(x,z) $ are parameters related to the loss function.
For more details, see the proof in \cite{Chen2018}.
\begin{lem}
	The discrimination learning curve $ \bar{\Gamma}(\hat{y},K) :=|\bar{\gamma_{0}}(\hat{y},K) - \bar{\gamma_{1}}(\hat{y},K)| $ is asymptotic and behaves as an inverse power law curve, where $ K $ is the size of the training set  \cite{Chen2018}.
\end{lem}
\begin{Theo}
	Unlabeled data help to reduce discrimination with our proposed method,  if $ (| \bar{V}_{z} (\hat{y})_{sl}|-| \bar{V}_{z} (\hat{y})_{ssl}|)  -  \bar{N}_{z,u} \geq 0$.
\end{Theo}
\begin{proof}
	To prove Theorem 1, the goal is to prove that the discrimination level in SSL $\bar{\Gamma}_{ssl}$ is lower than the discrimination level in supervised learning $\bar{\Gamma}_{sl}$, i.e., $ \bar{\Gamma}_{ssl} \leq \bar{\Gamma}_{sl}  $. In the following, we compare the discrimination in supervised learning and SSL in terms of $\emph{discrimination in bias}$ $\bar{B_z} (\hat{y})_{ssl}$, $\emph{discrimination in variance}$ $\bar{V_z} (\hat{y})_{ssl}$, $ \emph{and discrimination in noise} $ $\bar{N}_{z,ssl}$.
	
	\textbf{Discrimination in bias:}
	Bias can measure the fitting ability of the algorithm itself, and represent the accuracy of the model. Hence, bias in discrimination $ \bar{B_z} (\hat{y}) =\mathbbm{E}_{D} [B(y_{m},x,z)|Z=z]  $ only depends on the model. When the labeled dataset and mixed dataset are trained with the same model for both supervised and SSL settings, this can be expressed as $ | \bar{B} (\hat{y})_{sl}| - | \bar{B} (\hat{y})_{ssl}| = 0 $.
	
	\textbf{Discrimination in variance:}
	$ \textbf{Lemma 2} $ states that the discrimination level $ \bar{\Gamma}(\hat{y},n) $ decreases as the size of training data $ n $ increases. This means that discrimination in variance $ \bar{V_{z}}(f) $ can be lessened by more unlabeled data. 
	In our framework, the classifier is trained with labeled data and unlabeled data which is marked labels through label propagation. The size of the mixed training dataset $ K $ is larger than the size of the labeled training dataset $ K_{l} $. 
	Hence, we can conclude that $ | \bar{V}_{z} (\hat{y})_{ssl}| - | \bar{V}_{z} (\hat{y})_{sl}| \leq  0$.
	
	\textbf{Discrimination in noise:}
	Unlabeled data introduces discrimination in noise because predicting labels for unlabeled data via label propagation contains discrimination in noise. In our method, the fairness constraint $ c $ is related to the noise level in unlabeled data. A smaller $ c $ enables that label propagation is towards in a fairer direction. This means that smaller $ c $ brings smaller discrimination in noise. To analyze discrimination in noise in SSL, we divide it into discrimination in noise in labeled data $\bar{N}_{z,l} $ and discrimination in noise in unlabeled data $ \bar{N}_{z,u} $, which is expressed as,
	\begin{equation}
	\bar{N}_{z,ssl} = \bar{N}_{z,l} +\bar{N}_{z,u}
	\end{equation}
	Discrimination in noise in labeled data $ \bar{N}_{z,l} $ is the same as the discrimination in noise in supervised learning.  
	Discrimination in noise in unlabeled data may stem from four cases of mislabeled unlabeled data during label propagation,
	\begin{align}
	& \bar{N}_{y=0,z=0}=\mathbbm{E}_{D_{un}}[\hat{y}^*_{p}=1|y=0,z=0] \\
	& \bar{N}_{y=0,z=1}=\mathbbm{E}_{D_{un}}[\hat{y}^*_{p}=1|y=0,z=1] \\
	& \bar{N}_{y=1,z=0}=\mathbbm{E}_{D_{un}}[\hat{y}^*_{p}=0|y=1,z=0] \\
	& \bar{N}_{y=1,z=1}=\mathbbm{E}_{D_{un}}[\hat{y}^*_{p}=0|y=1,z=1] 
	\end{align}
	\noindent where $ \hat{y}_{p}^* $ is the optimal predicted label of unlabeled data via label propagation. With the `missing' labels filled in, the noise can be separated into the protected group $ \bar{N}_{1,u}=\bar{N}_{y=0,z=1} + \bar{N}_{y=1,z=1} $ and the unprotected group $ \bar{N}_{0,u}=\bar{N}_{y=0,z=0} + \bar{N}_{y=1,z=0} $. 
	In our framework, discrimination in  $ \bar{N}_{z,u} $ is controlled by the fairness constraint $ c $. A smaller $ c $ generates a smaller $ \bar{N}_{z,u} $.
	Hence,  $ \bar{N}_{z,u} $  can be adjusted by $ c $ and can be measured with,
	\begin{align}
	& \bar{N}_{z,u} = |\bar{N}_{1,u} - \bar{N}_{0,u}|
	\end{align}
	From the analysis above, we conclude that when  $ | \bar{V}_{z} (\hat{y})_{ssl} -    \bar{V}_{z} (\hat{y})_{sl}|- \bar{N}_{z,u} \geq 0$,  unlabeled data is able to reduce discrimination with our proposed method, which results in a better trade-off between accuracy and discrimination. Unlabeled data do not change discrimination in bias, reduce discrimination in variance and increase discrimination in noise. 
\end{proof}
Comparing with the discrimination analysis in \cite{zhang2020fairness} in SSL, discrimination in variance is reduced by unlabeled data and ensemble learning, and discrimination in noise is reduced by ensemble learning. Our proposed method can reduce discrimination in variance by unlabeled data and reduce discrimination in noise by the fairness constraint. The advantage of our method is that the discrimination in noise is controllable with the threshold $ c $ in the fairness constraint. \cite{zhang2020fairness} and our proposed method are complementary work to explore how to reduce the discrimination in noise that unlabeled data may bring.


\section{Experiment}
In this section, we first describe the experimental setup, including datasets, baselines, and parameters. Then, we evaluate our method on three real-world datasets under the fairness metric of disparate impact and disparate mistreatment (including OMR, FNR and FPR).
The aim of our experiments is to assess: the effectiveness of our method to achieve fair semi-supervised learning; the impact of different fairness constraints on fairness; and the extent to which unlabeled can balance fairness with accuracy. 
\subsection{Experimental Setup}
\subsubsection{Dataset}
Our experiments involve three real-world datasets: Health  dataset \footnote{https://foreverdata.org/1015/index.html}, Titanic dataset \footnote{https://www.kaggle.com/c/titanic/data} and Bank dataset \footnote{https://archive.ics.uci.edu/ml/datasets/bank+marketing}. 
\begin{itemize}
	\item The task in the Health dataset is to predict whether people will spend time in the hospital. In order to convert the problem into the binary classification task, we only  predict whether people will spend any day in the hospital. 	After data preprocessing, the dataset contains 27,000 data points with 132 features.
	We divide patients into two groups based on age ($ \geq $65 years)  and consider 'Age' to be the sensitive attribute.
	\item The Bank dataset contains a total of 41,188 records with 20 attributes and a binary label, which indicates whether the client has subscribed (positive class) or not (negative class) to a term deposit. We consider 'Age' as sensitive attribute.
	\item  The Titanic dataset comes from a Kaggle competition where the goal is to analyze which sorts of people were likely to survive the sinking of the Titanic. We consider "Gender" as the sensitive attribute. After data preprocessing, we extract 891 data points with 9 features.
\end{itemize}

\subsubsection{Parameters}
The sensitive attributes are excluded from the training set to ensure fairness between groups and are only used to evaluate discrimination in the test phrases.
In the Health, Bank and Titanic datasets, data are all labeled. 
In the Health dataset, we sample 4000 data points as labeled dataset, 4000 data points as test dataset, and left as unlabeled dataset. In the Bank dataset, we sample 4000 data points as labeled dataset, 4000 data points as test dataset, and left as unlabeled dataset.  In the Titanic dataset, we sample 200 data points as labeled dataset, 200 data points as test dataset, and left as unlabeled dataset. 
Therefore, $\mathcal{D}_{l}$ and $\mathcal{D}_{u} $ are collected from the similar data distribution. 

In the experiments, the results are an average of 10 results by randomly sampling labeled dataset, test dataset and unlabeled dataset.  
We set $ \alpha=1 $ and $ T= 0.5 $ in all datasets; $ \sigma =0.5  $ in the Health dataset and Bank dataset, and  $ \sigma =0.1  $ in the Titanic dataset.  
In DCCP, parameter $ \tau $ is set to $0.05  $ and $ 1 $ in the Bank and Titanic dataset, respectively. Parameter $ \mu $ is set to $1.2  $ in both of Bank and Titanic datasets.

\subsubsection{Baseline Methods}
The methods chosen for comparison are listed as follows. 
It is worth to note that \cite{NIPS2019_9437} also used unlabeled data on fairness. However, they only applied the equal opportunity metric, which is different to ours. Hence, we did not compare the proposed method with them.
\begin{itemize}
	\item Fairness Constraints (FS): Fairness constraints are used to ensure fairness for classifiers. \cite{Zafar2017}
	\item Uniform Sampling (US): The number of data points in each groups is equalized  through oversampling and/ undersampling.  \cite{Kamiran2012}
	\item Preferential Sampling (PS): The number of data points in each groups is equalized by taking samples near the borderline data points. \cite{Kamiran2012}
	\item Fairness-enhanced sampling (FES): A fair SSL framework includes pseudo labeling, re-sampling and ensemble learning. \cite{zhang2020fairness} 
\end{itemize}

\subsection{Experimental Results of disparate impact}

\begin{figure}[ht]
	\begin{minipage}[b]{0.49\linewidth}
		\centering	
		\includegraphics[scale=0.24]{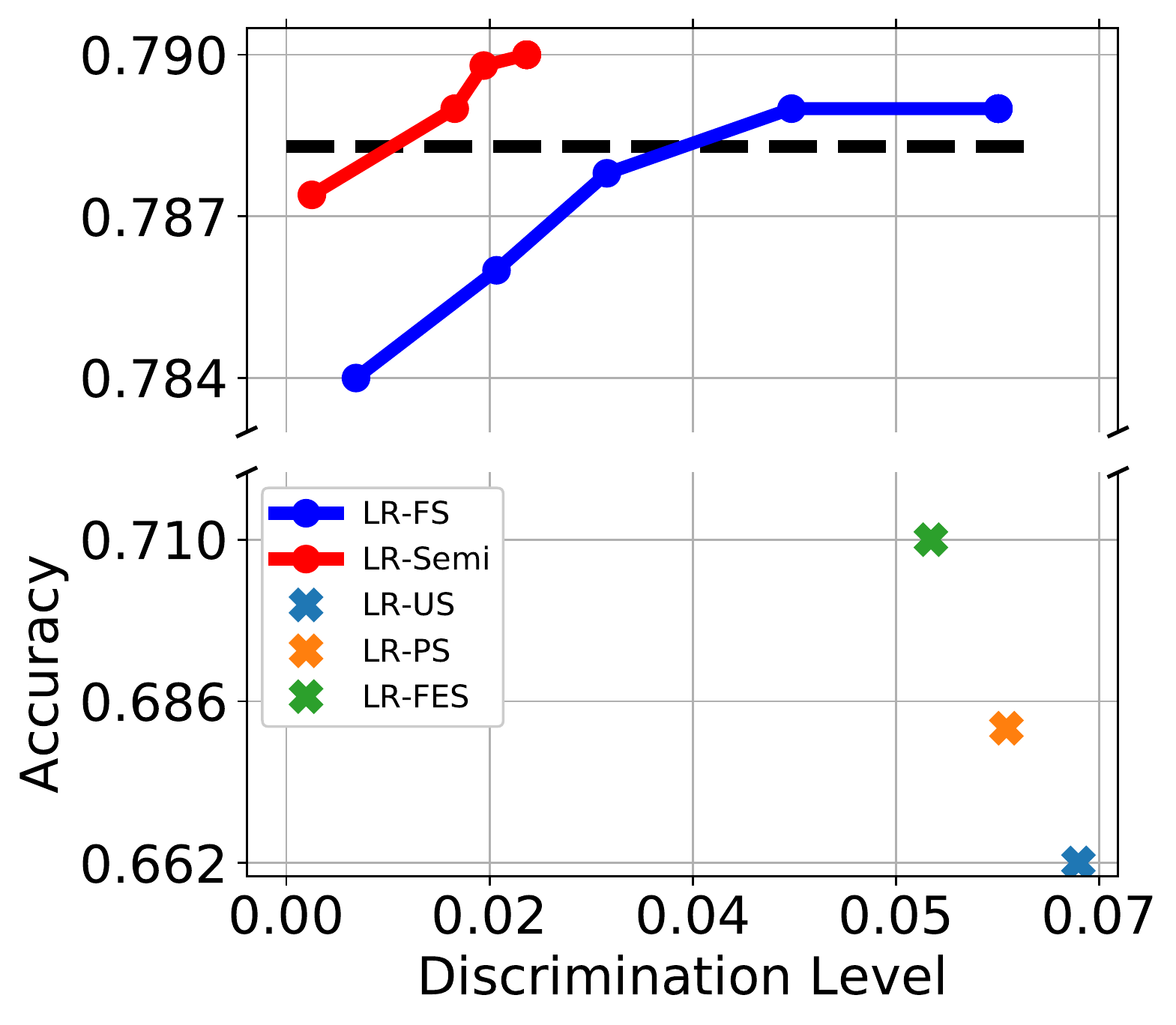}
		\centerline{(a) LR-Health}
	\end{minipage}
	\begin{minipage}[b]{0.49\linewidth}
		\centering	
		\includegraphics[scale=0.24]{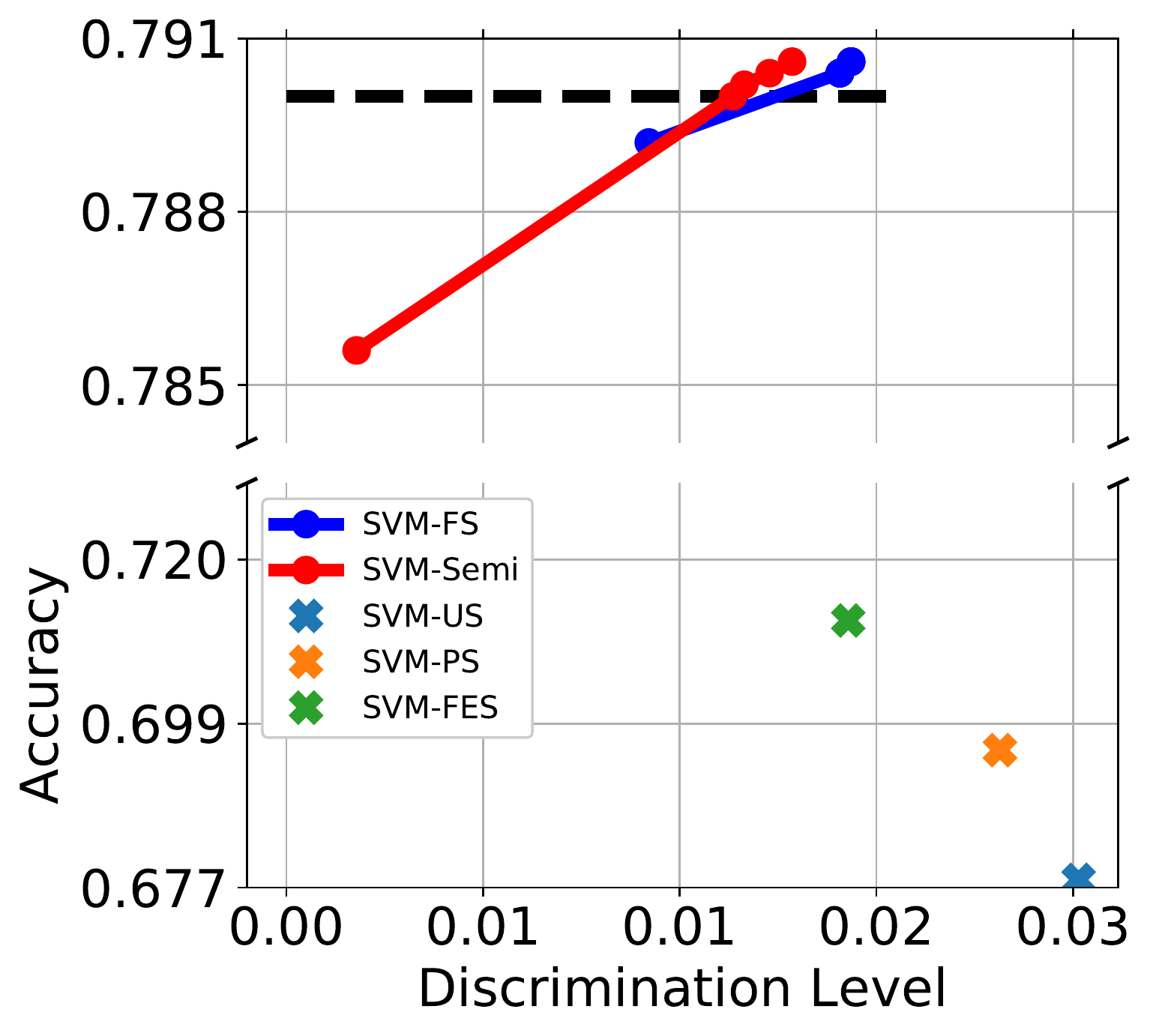}
		\centerline{(b) SVM-Health}
	\end{minipage}
	\begin{minipage}[b]{0.49\linewidth}
		\centering	
		\includegraphics[scale=0.24]{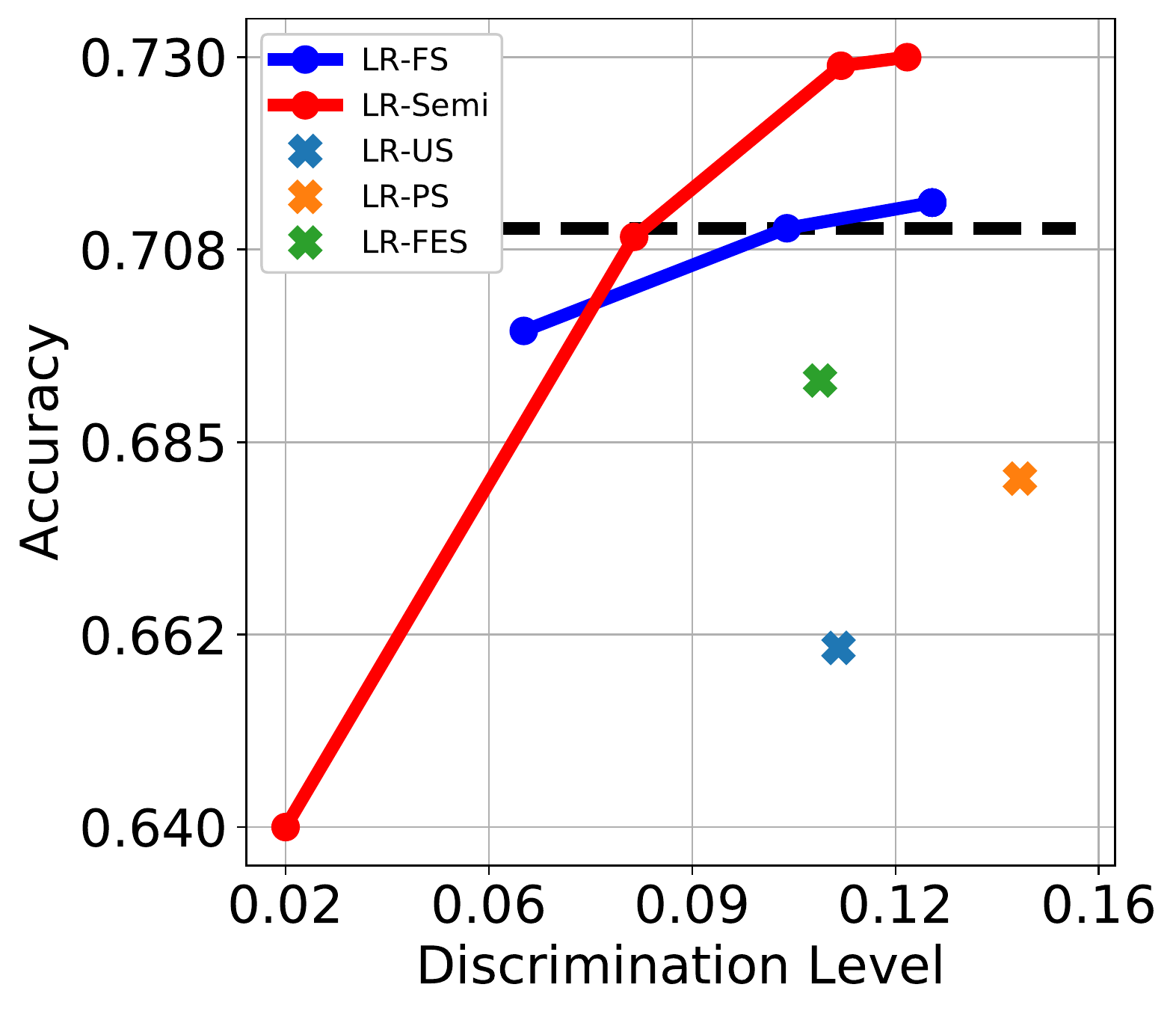}
		\centerline{(c) LR-Titanic}
	\end{minipage}
	\begin{minipage}[b]{0.49\linewidth}
		\centering	
		\includegraphics[scale=0.24]{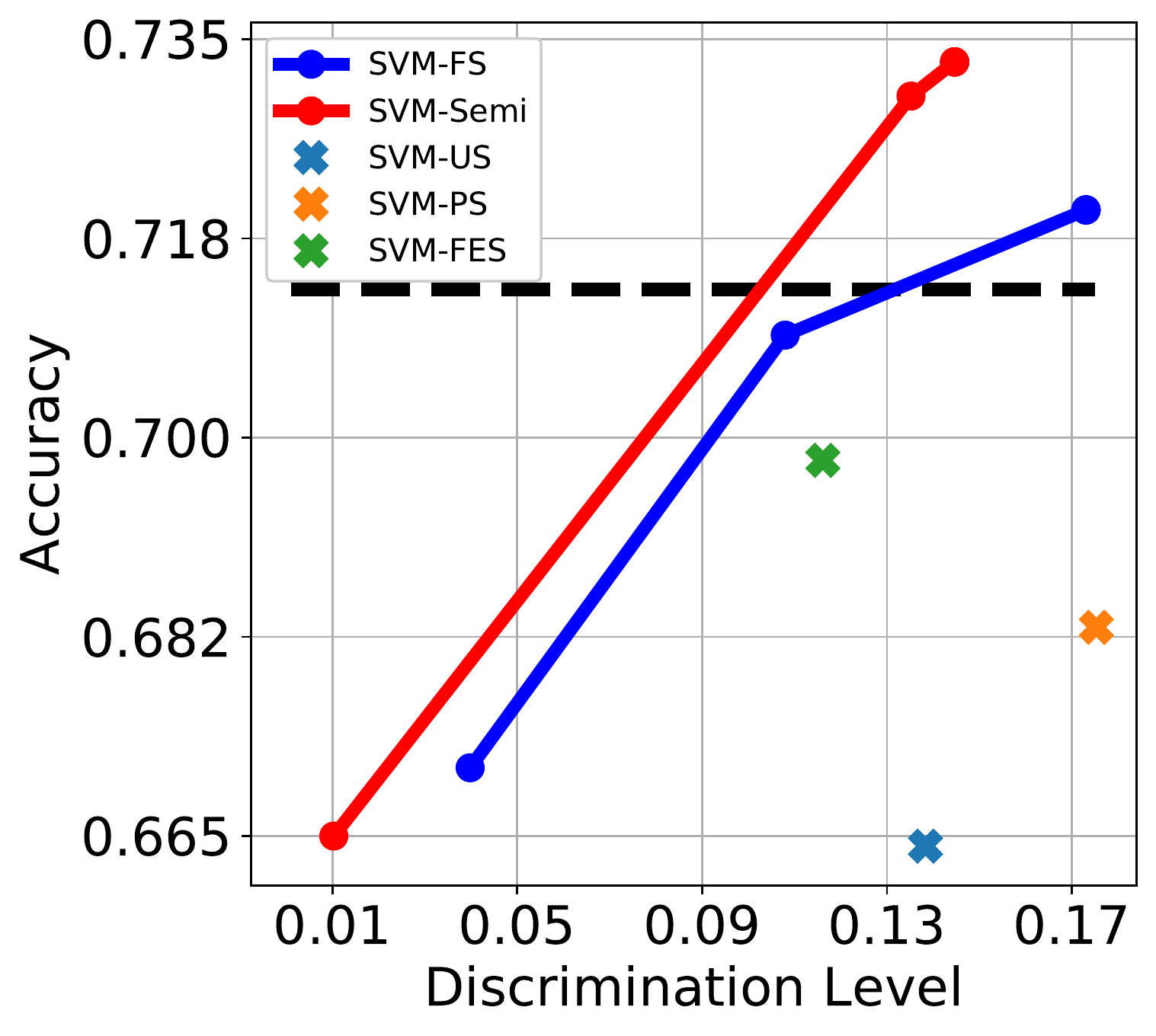}
		\centerline{(d) SVM-Titanic}
	\end{minipage}
	\caption{The trade-off between accuracy and discrimination in proposed method Semi (Red), FS (Blue), US (Blue cross), PS (Yellow cross) and FES (Green cross) under the fairness metric of disparate impact with LR and SVM in two datasets. As the threshold of covariance $ c $ increases, accuracy and discrimination increase. The results demonstrate that and our method achieves a better trade-off between accuracy and discrimination than other methods. }
\end{figure}

\subsubsection{Trade-off Between Accuracy and Discrimination}
Figure 1 shows that as $ c $ varies, accuracy and discrimination level in the proposed method and other methods with LR and SVM on two datasets. From the results, we can observe that our framework provides the better trade-off between accuracy and discrimination. A better trade-off means that with the same accuracy, discrimination is low or with the same discrimination, accuracy is higher.
For example, at the same level of accuracy on the Titanic dataset, (shown by the black line),  our method with LR has a discrimination level of around $ 0.08 $, while FS method has a discrimination level of $ 0.11 $. A similar observation can be made from the results with PS method (Yellow cross), US method (Blue cross) and FES method (Green cross).
Note that the discrimination level (red line) with LR in the Health dataset does not extend because discrimination does not increase as $ c $ grows.
Additionally, we note that accuracy and discrimination level are related to the training models. In the Titanic dataset, LR has a lower accuracy and discrimination than SVM and the choice of training models is related to the datasets.
\begin{table}[]
	\scalebox{0.78}{
		\begin{tabular}{lcccccccc} \hline
			Dataset    & \multicolumn{7}{c}{Health dataset}                                                                      &             \\ \hline
			Constraint & \multicolumn{2}{l}{Labeled} & \multicolumn{2}{l}{Unlabeled} & \multicolumn{2}{l}{Combined} & \multicolumn{2}{l}{Mixed} \\ \hline
			& Acc          & Dis          & Acc           & Dis           & Acc           & Dis          & Acc         & Dis         \\             \hline
			c=0.0      & 0.7868       & 0.0042       & N/A           & N/A           & {0.7874}        & 0.0022       & 0.7862      & {0.0003}      \\ 
			c=0.1      & 0.7890       & {0.0129}       & N/A           & N/A           & 0.7890        & 0.0145       &{0.7892}      & 0.0149      \\
			c=0.2      &{0.7900}       & 0.0170       & 0.7900        & 0.0207        & 0.7898        & {0.0170}       & 0.7898      & 0.0170      \\
			c=0.3      & 0.7898       & 0.0207       & 0.7898        & 0.0170        & 0.7900        & 0.0207       & {0.7900}      & {0.0207}      \\
			c=0.4      & {0.7902}       & 0.0178       & 0.7898        & {0.0170}        & 0.7900        & 0.0207       & 0.7900      & 0.0207      \\
			c=0.5      & 0.7900       & 0.0207       & 0.7900        & 0.0207        & 0.7900        & 0.0207       & {0.7900}      & {0.0207}      \\
			c=0.6      & 0.7900       & 0.0207       & {0.7906}        & {0.0186}        & 0.7900        & 0.0207       & 0.7900      & 0.0207      \\
			c=0.7      & 0.7900       & 0.0207       & 0.7900        & 0.0207        & 0.7900        & 0.0207       & {0.7900}      & {0.0207}      \\
			c=0.8      & 0.7900       & 0.0207       & {0.7904}        & {0.0191}        & 0.7900        & 0.0207       & 0.7900      & 0.0207      \\
			c=0.9      & 0.7900       & 0.0207       & {0.7908}        & {0.0190}        & 0.7900        & 0.0207       & 0.7900      & 0.0207      \\
			c=1.0      & 0.7900       & 0.0207       & 0.7900        & 0.0207        & 0.7900        & 0.0207       & {0.7900}      & {0.0207}     \\ \hline
		\end{tabular}%
	}
	\caption{The impact of fairness constraints on different datasets in terms of accuracy (Acc) and discrimination level (Dis) under the fairness metric of disparate impact with LR in the Health dataset.}
	\label{tab:my-table}
\end{table}

\begin{table}[]
	\scalebox{0.78}{
		\begin{tabular}{lcccccccc}
			\hline
			Dataset    & \multicolumn{8}{c}{Titanic dataset}                                                                                   \\ \hline
			Constraint & \multicolumn{2}{c}{Labeled} & \multicolumn{2}{c}{Unlabeled} & \multicolumn{2}{c}{Combined} & \multicolumn{2}{c}{Mixed} \\ \hline
			& Acc      & Dis          & Acc          & Dis           & Acc           & Dis           & Acc         & Dis         \\ \hline
			c=0.0      & 0.6330       & {0.0128}       & {0.6970}        & 0.1244        & 0.6290        & 0.0139       & 0.6440      & 0.0402      \\
			c=0.05     & 0.6690       & {0.0579}       & {0.7070}        & 0.1265        & 0.6690        & 0.0716       & 0.6810      & 0.0948      \\
			c=0.1      & 0.7150       & 0.1272       & 0.7140        & 0.1332        & 0.7100        & 0.1239       & {0.7150}      & {0.1256}      \\
			c=0.15     & 0.7200       & 0.1366       & 0.7190        & 0.1336        & 0.7190        & 0.1336       & {0.7200}      & {0.1366}      \\
			c=0.2      & 0.7200       & 0.1366       & 0.7200        & 0.1366        & 0.7200        & 0.1366       & {0.7200}      & {0.1366}      \\
			c=0.25     & 0.7200       & 0.1366       & 0.7200        & 0.1366        & 0.7200        & 0.1366       & {0.7200}      & {0.1366}    \\ \hline
		\end{tabular}%
	}
	\caption{The impact of fairness constraints on different datasets in terms of accuracy (Acc) and discrimination level (Dis) under the fairness metric of disparate impact with LR in the Titanic dataset.}
	\label{tab:my-table}
\end{table}


\subsubsection{Different Fairness Constraints}
Our next set of experiments is to determine the impact of different fairness constraints. For these tests, the size of unlabeled data is set to 12,000 data points in the Health dataset and 400 data points in the Titanic dataset. Due to space limitation, we have only reported the results for the LR, which appear in Tables 1 and 2. The result show that, when  varying the threshold of covariance $ c $,  different fairness constraints on labeled and unlabeled data have different impacts on the training results.
As the threshold of covariance increases, both accuracy and discrimination level increase before steadying off for the duration. In terms of accuracy, this is because a larger $ c $ allows for a larger space to find better weights $ \boldsymbol{w} $  to inform classification. In terms of discrimination, a larger $ c $ tends to introduce more discrimination in noise.

It is also observed that the fairness constraint on mixed data generally has the best performance in the trade-off between accuracy and discrimination. Other three constraints have very similar accuracy and discrimination levels. We attribute this to the assumption that labeled and unlabeled data have the similar data distribution, and therefore the mixed fairness constraint on labeled and unlabeled data gives the best description of the covariance between sensitive attributes and signed distance from feature vectors to the decision boundary. 
\begin{figure}[ht]
	\begin{minipage}[b]{0.49\linewidth}
		\centering	
		\includegraphics[scale=0.22]{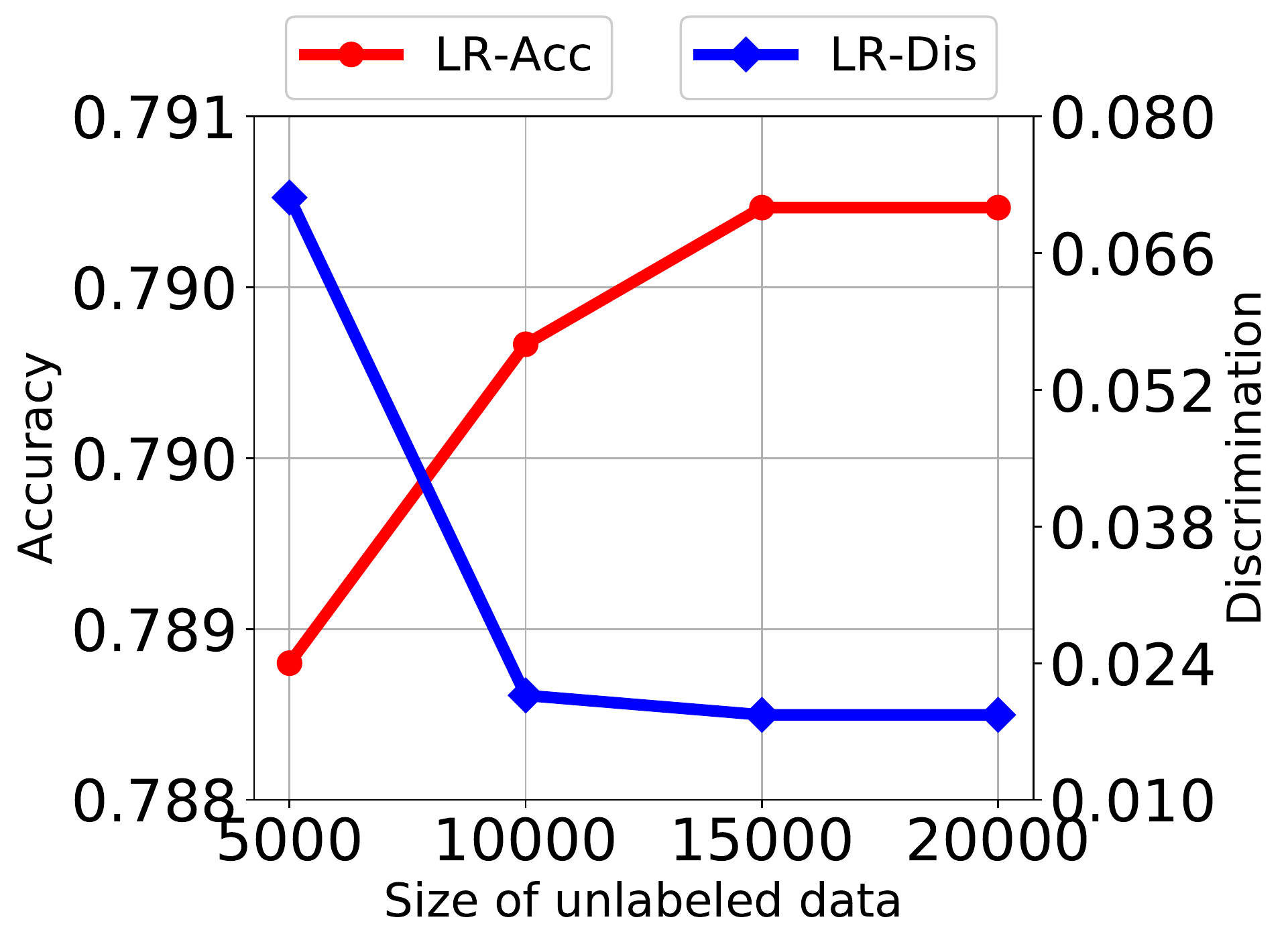}
		\centerline{(a) LR-Health}
	\end{minipage}
	\begin{minipage}[b]{0.49\linewidth}
		\includegraphics[scale=0.22]{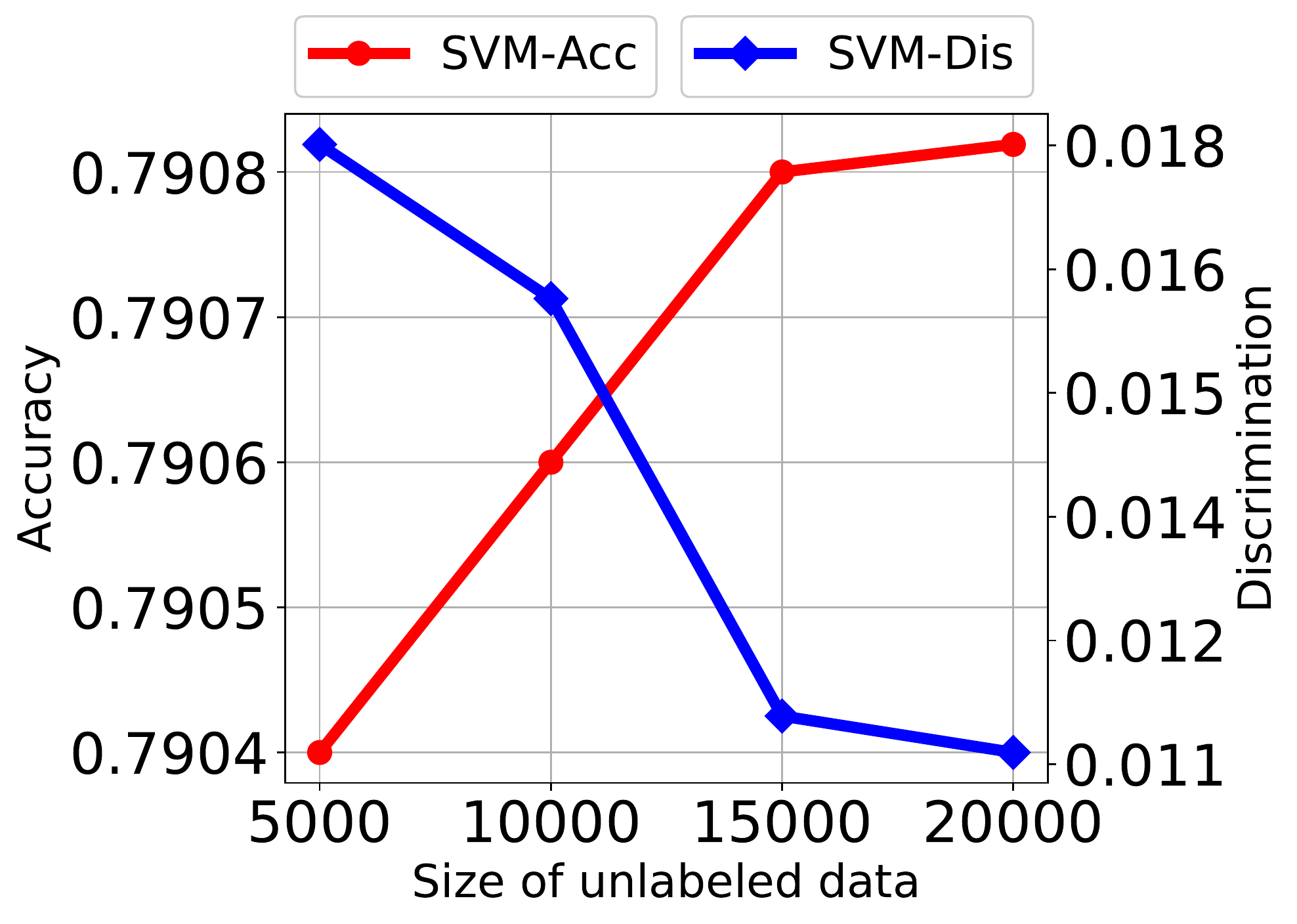}
		\centerline{(b) SVM-Health}
	\end{minipage}
	\begin{minipage}[b]{0.49\linewidth}
		\includegraphics[scale=0.22]{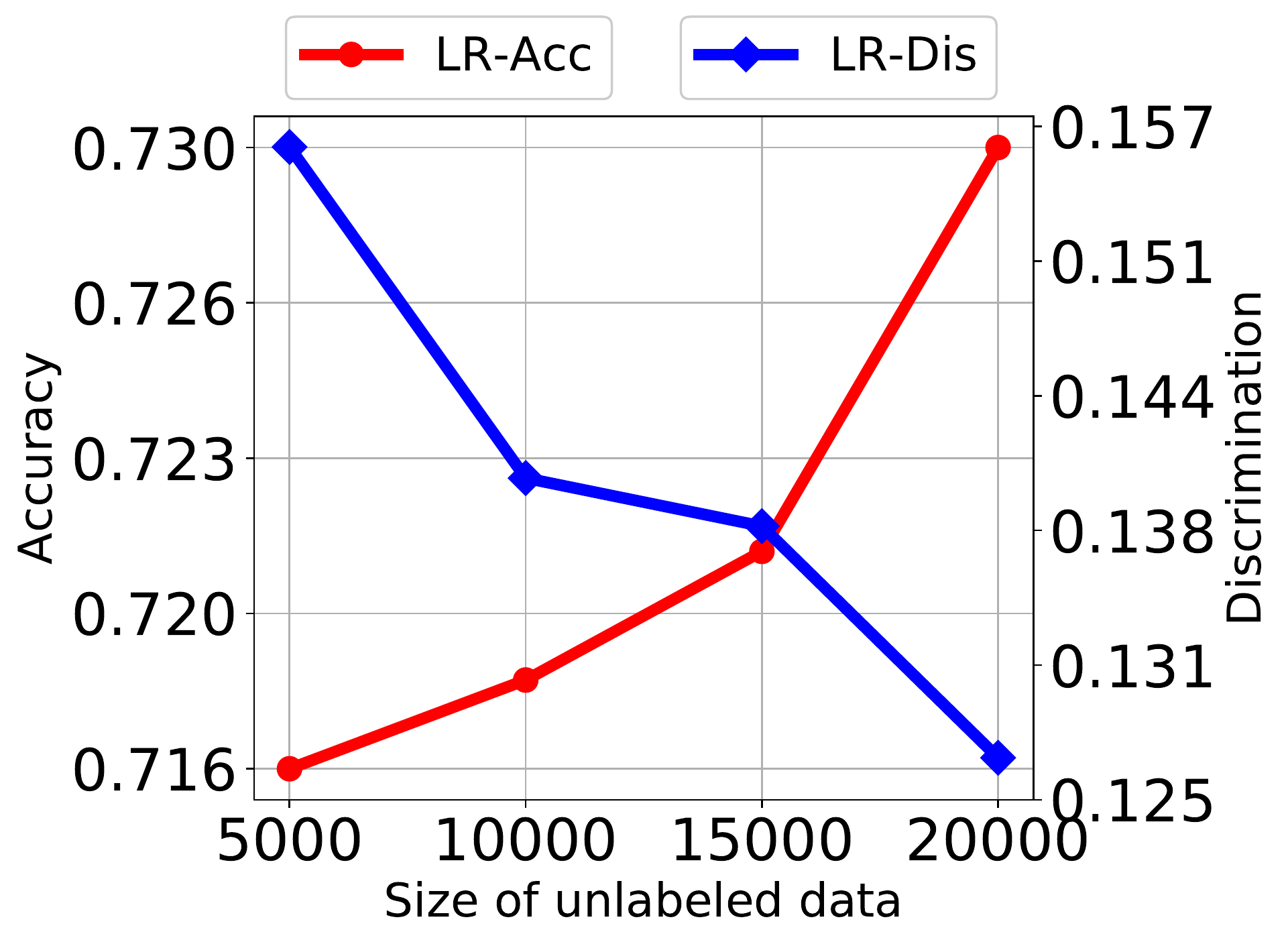}
		\centerline{(c) LR-Titanic}
	\end{minipage}
	\begin{minipage}[b]{0.49\linewidth}
		\includegraphics[scale=0.22]{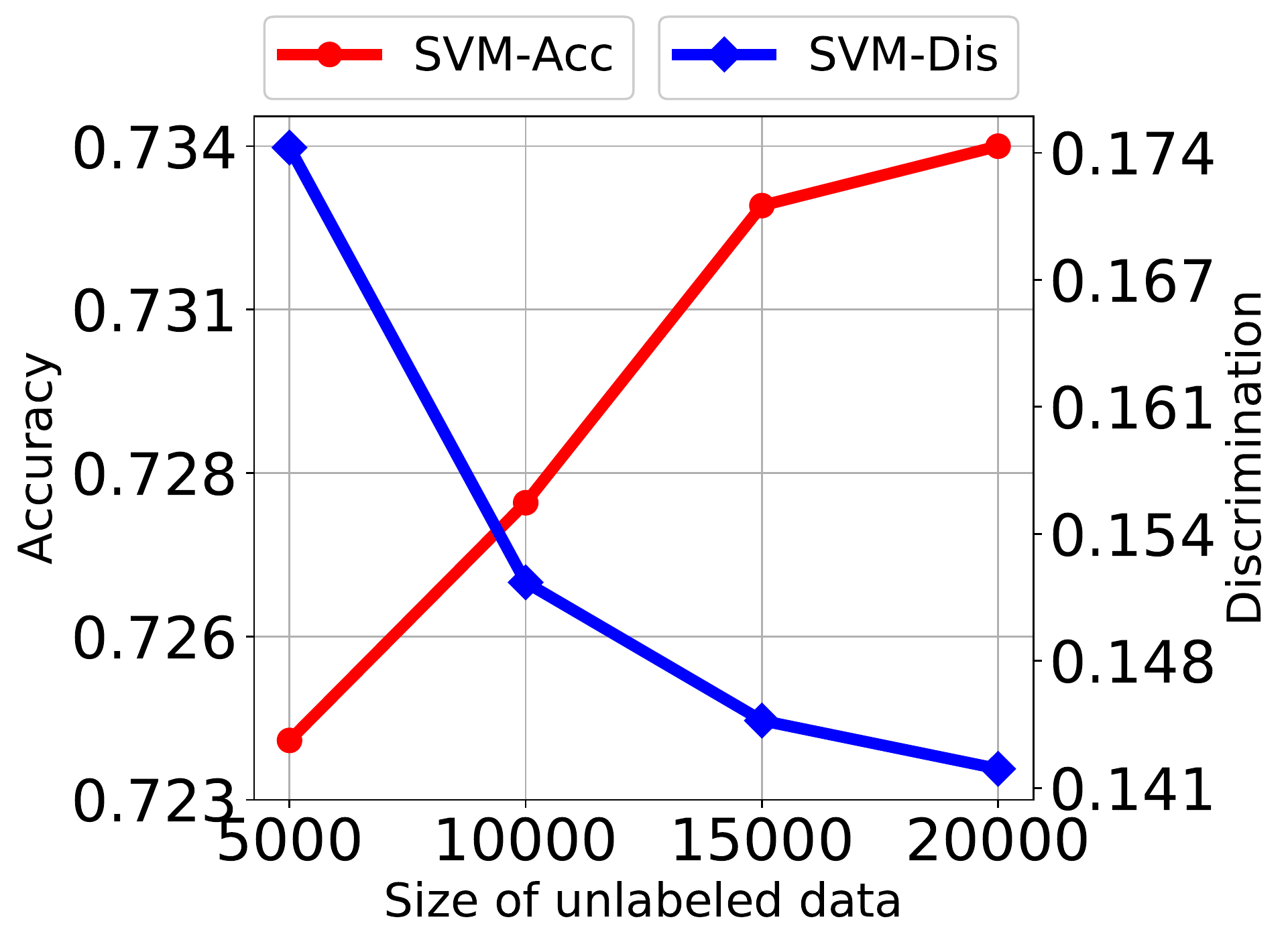}
		\centerline{(d) SVM-Titanic}
	\end{minipage}
	\caption{The impact of the amount of unlabeled data in the training set on accuracy (Red) and discrimination level (Blue) under the fairness metric of disparate impact with LR and SVM in two datasets. The X-axis is the size of unlabeled dataset; left y-axis is accuracy; and right y-axis is discrimination level.}
	
\end{figure}
\subsubsection{The Impact of Unlabeled Data}
For these experiments, we set the covariance threshold $ c=1 $ for the Health and Titanic datasets.
Figure 2 shows that accuracy and discrimination level varies with the amount of unlabeled data. This applies to both the LR and SVM classifiers on both datasets.
As shown, accuracy increases as the amount of unlabeled data increases in both datasets before stabilizing at its peak. Discrimination level sharply decreases almost immediately, then also stabilize. These results clearly demonstrate that discrimination in variance decreases as the amount of unlabeled data in the training set increases.

\subsection{Experimental Results of Disparate Mistreatment}
\subsubsection{Trade-off Between Accuracy and Discrimination}
\begin{figure}[ht]
	\begin{minipage}[b]{0.49\linewidth}
		\centering	
		\includegraphics[scale=0.24]{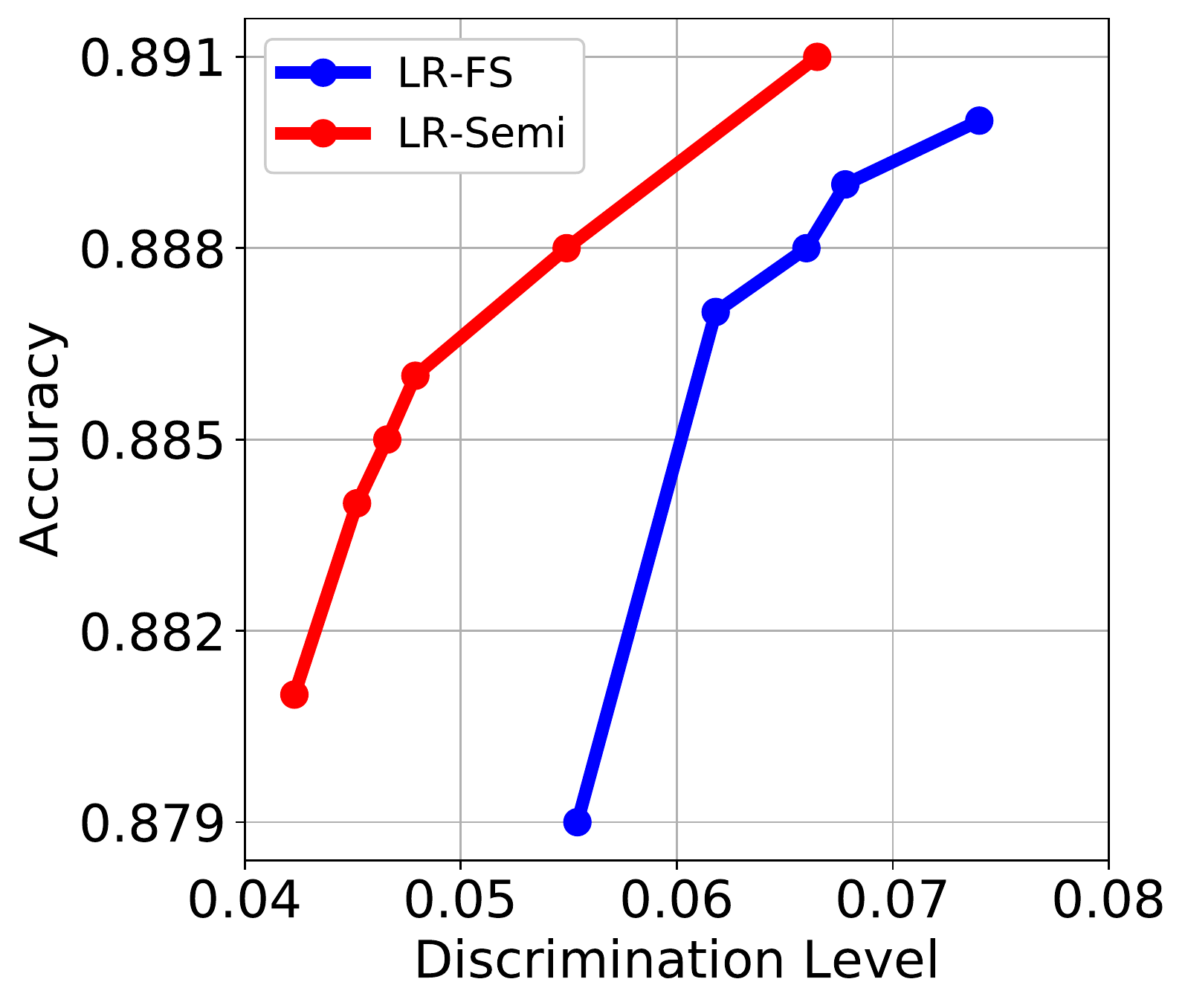}
		\centerline{(a) LR-Bank-OMR}
	\end{minipage}
	\begin{minipage}[b]{0.49\linewidth}
		\centering
		\includegraphics[scale=0.24]{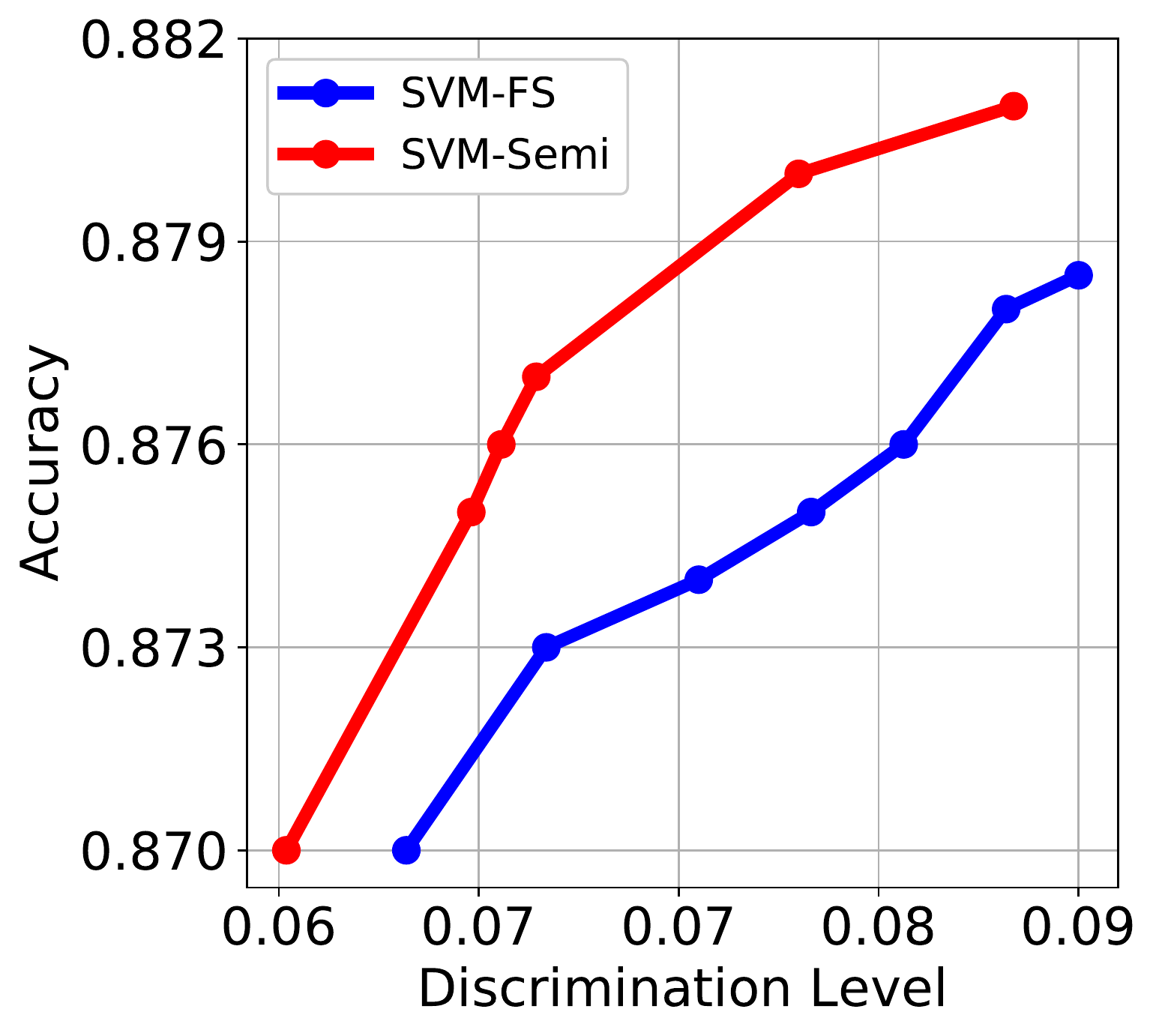}
		\centerline{(b) SVM-Bank-OMR}
	\end{minipage}
	\begin{minipage}[b]{0.49\linewidth}
		\centering
		\includegraphics[scale=0.24]{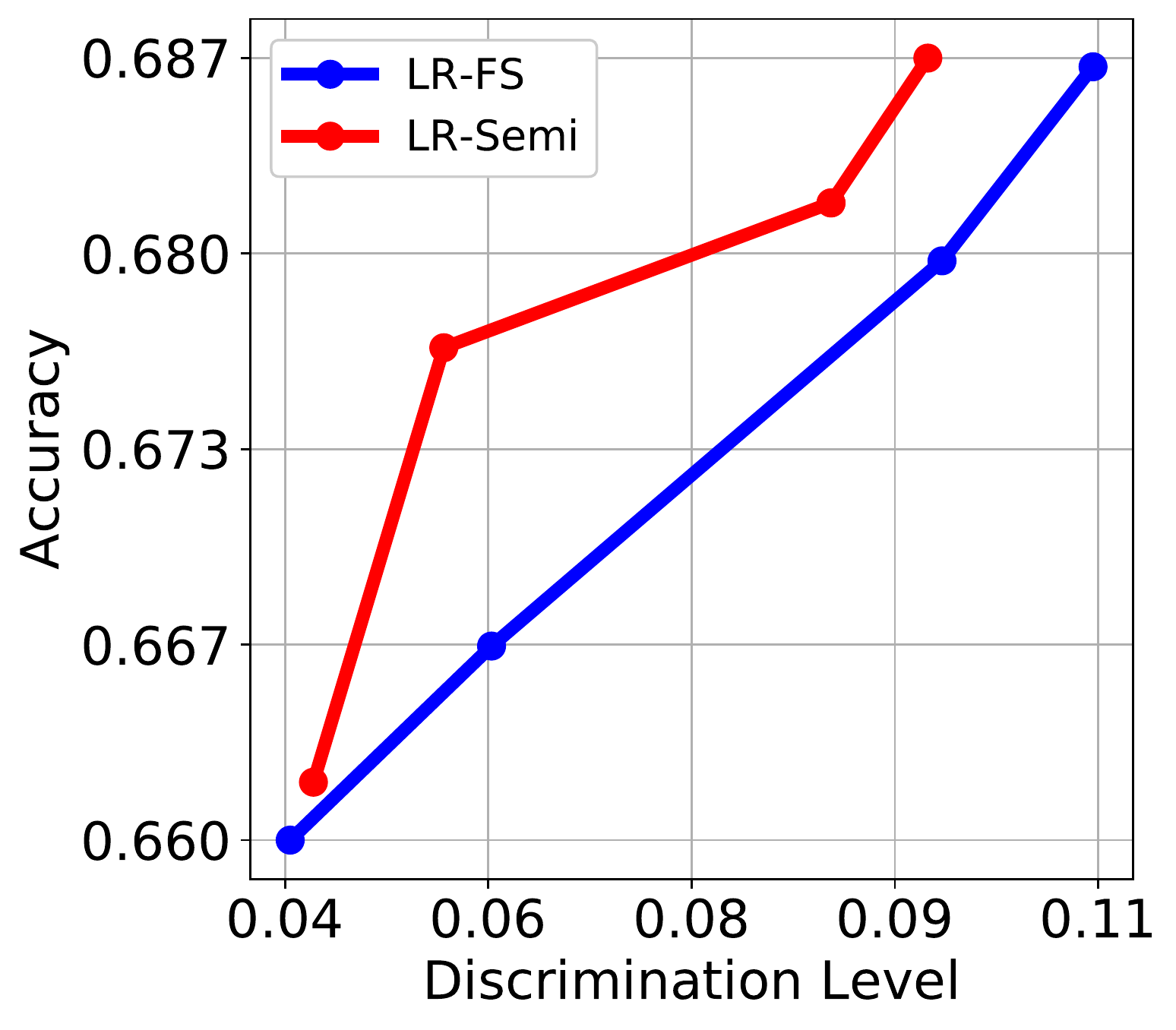}
		\centerline{(c) LR-Titanic-OMR}
	\end{minipage}
	\begin{minipage}[b]{0.49\linewidth}
		\centering
		\includegraphics[scale=0.24]{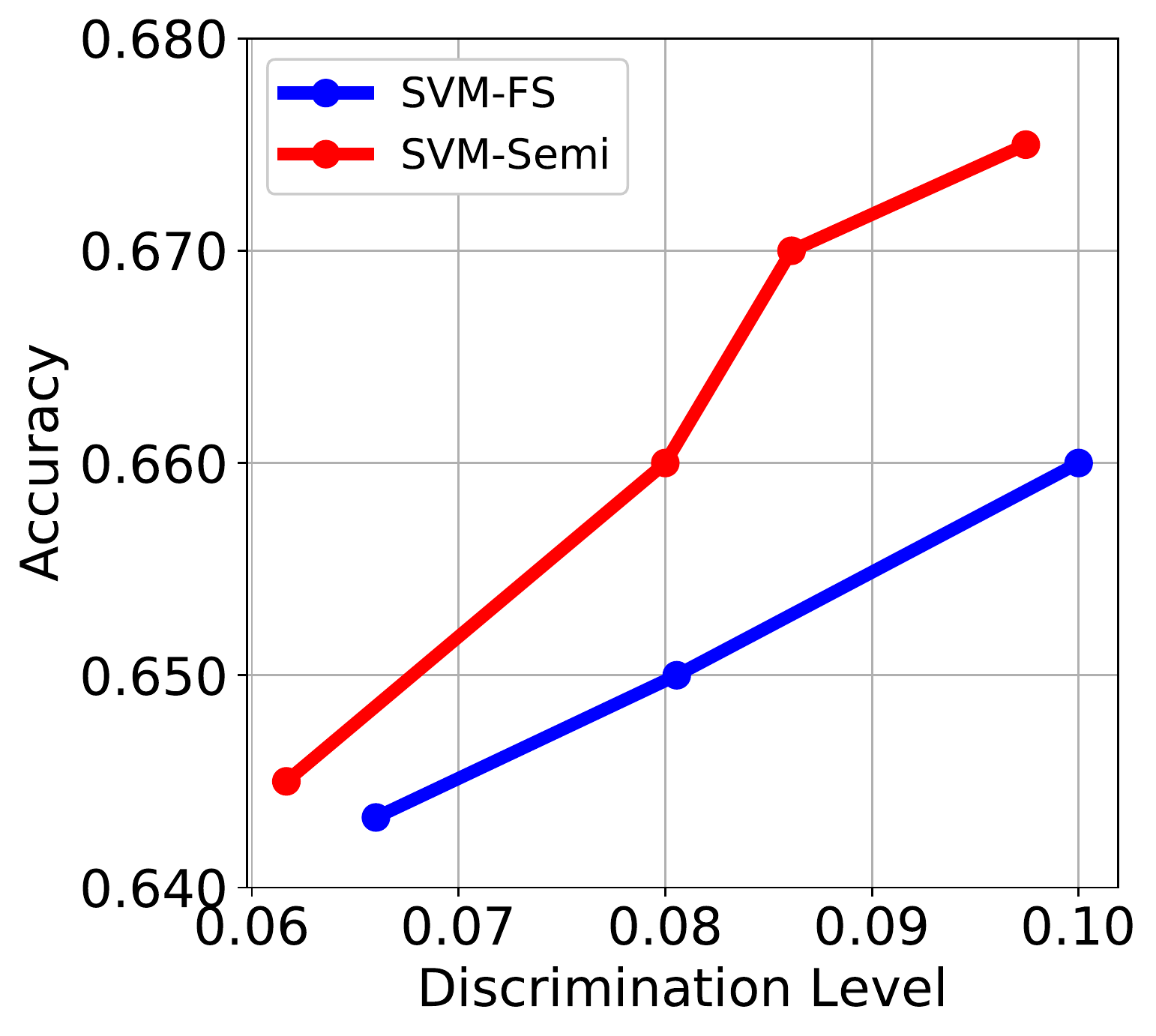}
		\centerline{(d) SVM-Titanic-OMR}
	\end{minipage}
	\caption{The trade-off between accuracy and discrimination in proposed method Semi (Red), FS (Blue) with LR and SVM in two datasets under the metric of overall misclassification rate. As the threshold of covariance $ c $ increases, accuracy and discrimination increase. The results demonstrate that our method using unlabeled data achieves a better trade-off between accuracy and discrimination.}	
\end{figure}
Figures 3-5 show that as $ c $ varies, accuracy and discrimination level in the proposed framework and the FS method with LR and SVM on two datasets under the fairness metric of OMR, FPR and FNR. From the results, we can observe that our proposed method (Red line) generally is in the left above the FS method (Blue line).
This indicates that our framework provides the better trade-off between accuracy and discrimination in three metrics for the most time. 
For example, at the same level of accuracy (Acc = $ 0.885 $) on the Bank dataset under OMR, our method with LR has a discrimination level of around $0.045$, while FS method has a discrimination level of $ 0.06 $. We also observe that discrimination level is quite different under fairness metrics. For example, discrimination level can reach $ 0.17 $ at the end under FNR, while discrimination level only shows $ 0.01 $ under FPR.
In addition, we note that accuracy and discrimination level have different performance on training models. In the Bank dataset, SVM generally has a lower accuracy and discrimination than LR.
\begin{figure}[ht]
	\begin{minipage}[b]{0.49\linewidth}
		\centering	
		\includegraphics[scale=0.24]{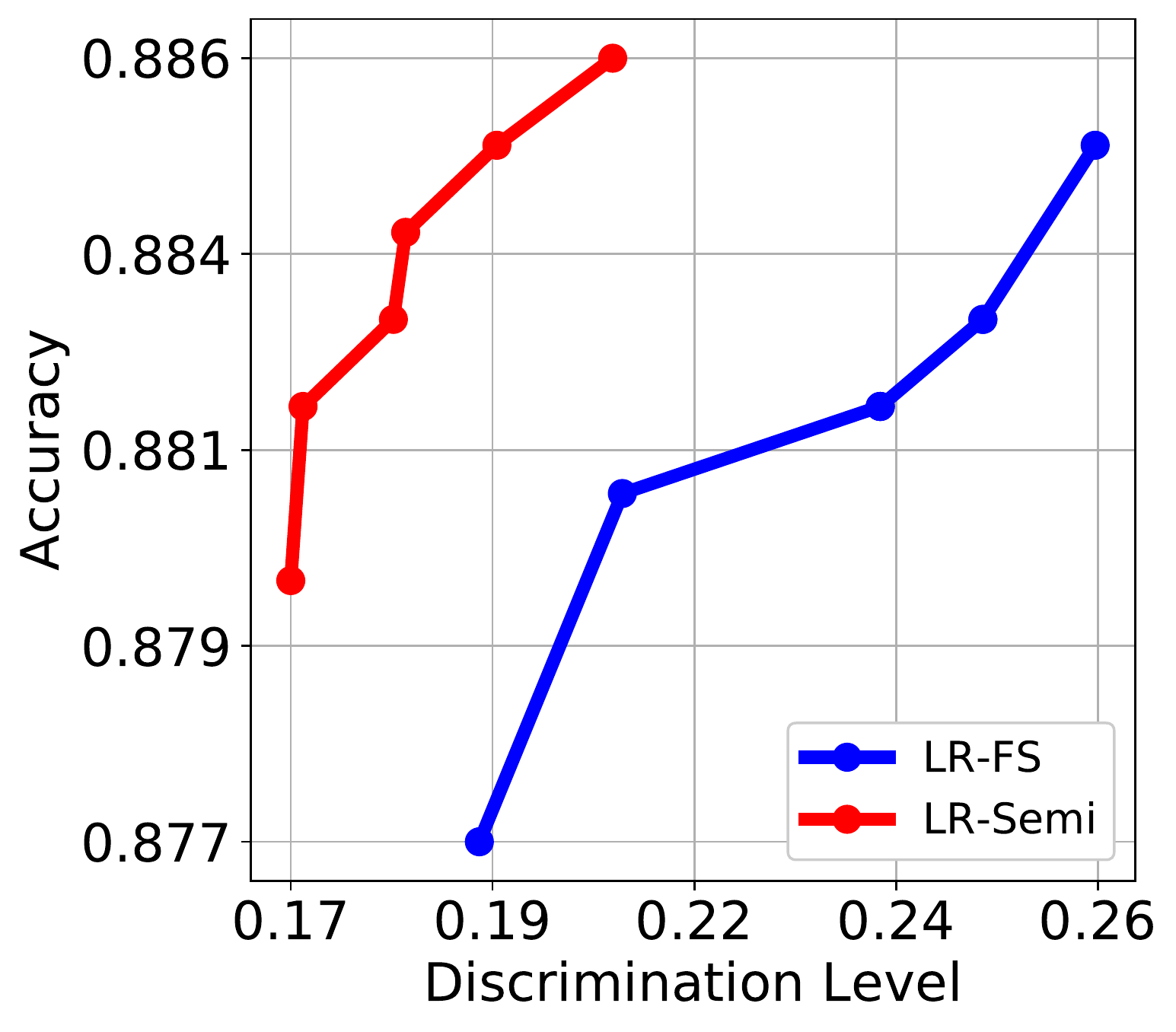}
		\centerline{(a) LR-Bank-FNR}
	\end{minipage}
	\begin{minipage}[b]{0.49\linewidth}
		\centering
		\includegraphics[scale=0.24]{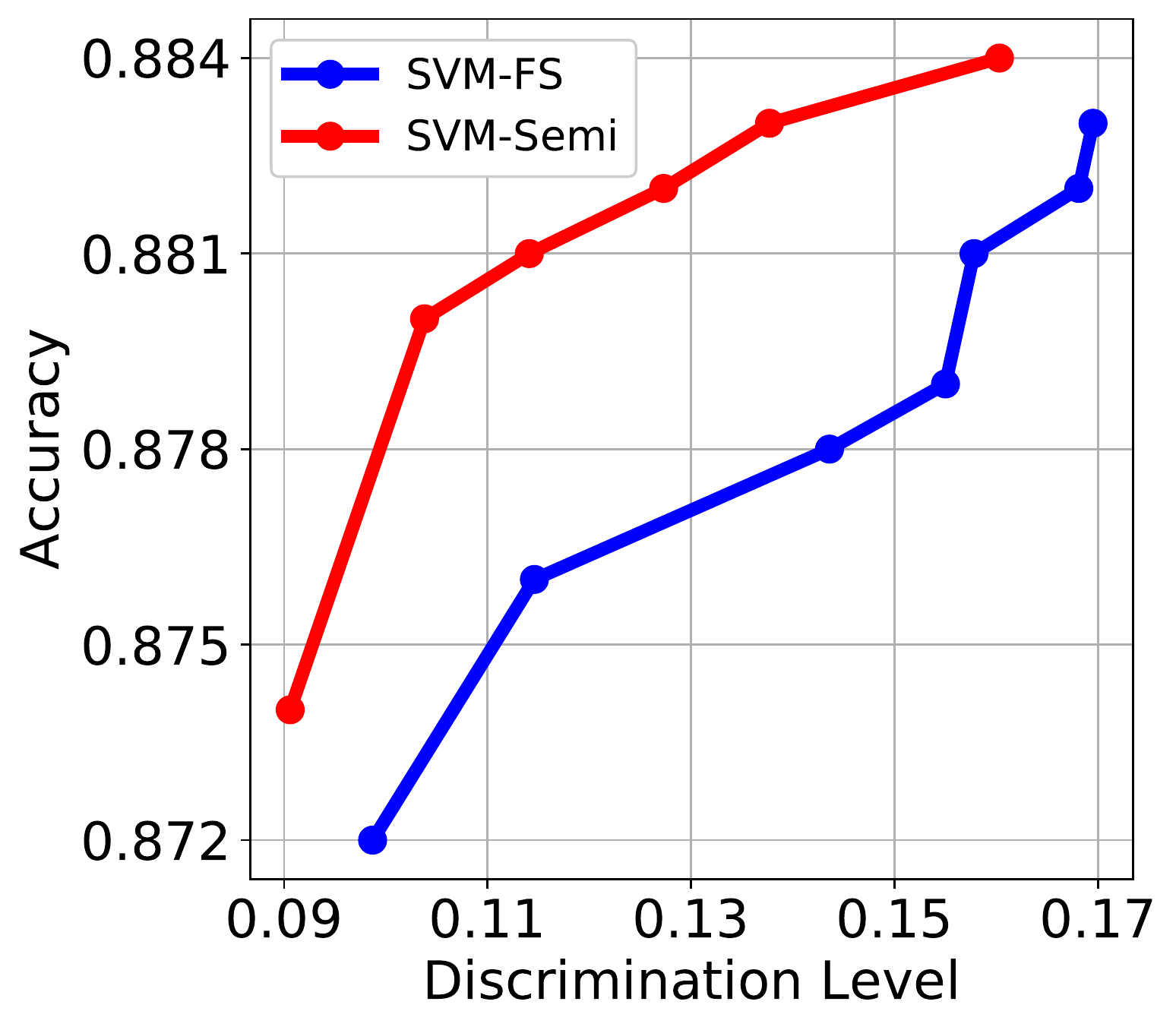}
		\centerline{(b) SVM-Bank-FNR}
	\end{minipage}
	\begin{minipage}[b]{0.49\linewidth}
		\centering
		\includegraphics[scale=0.24]{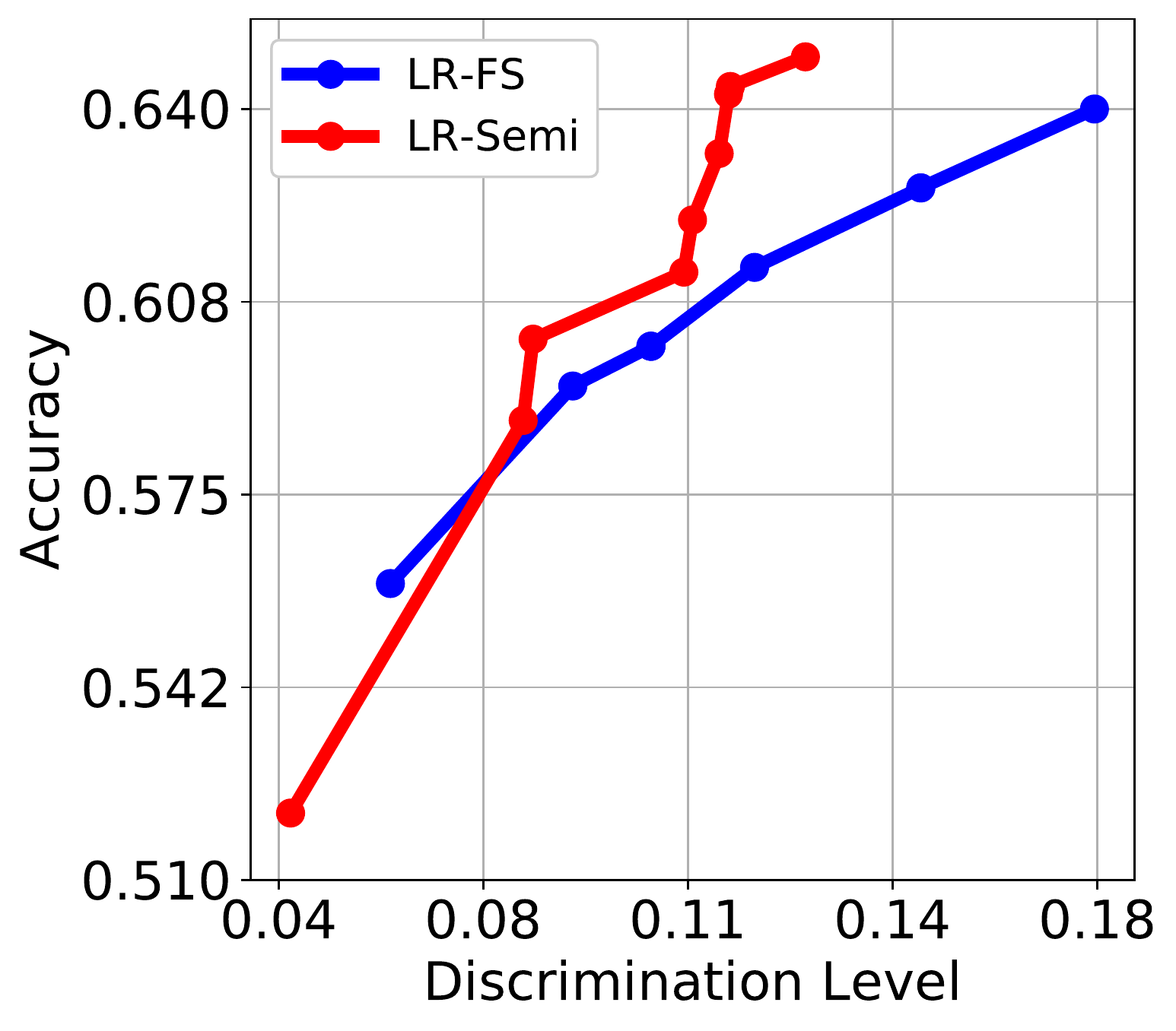}
		\centerline{(c) LR-Titanic-FNR}
	\end{minipage}
	\begin{minipage}[b]{0.49\linewidth}
		\centering
		\includegraphics[scale=0.24]{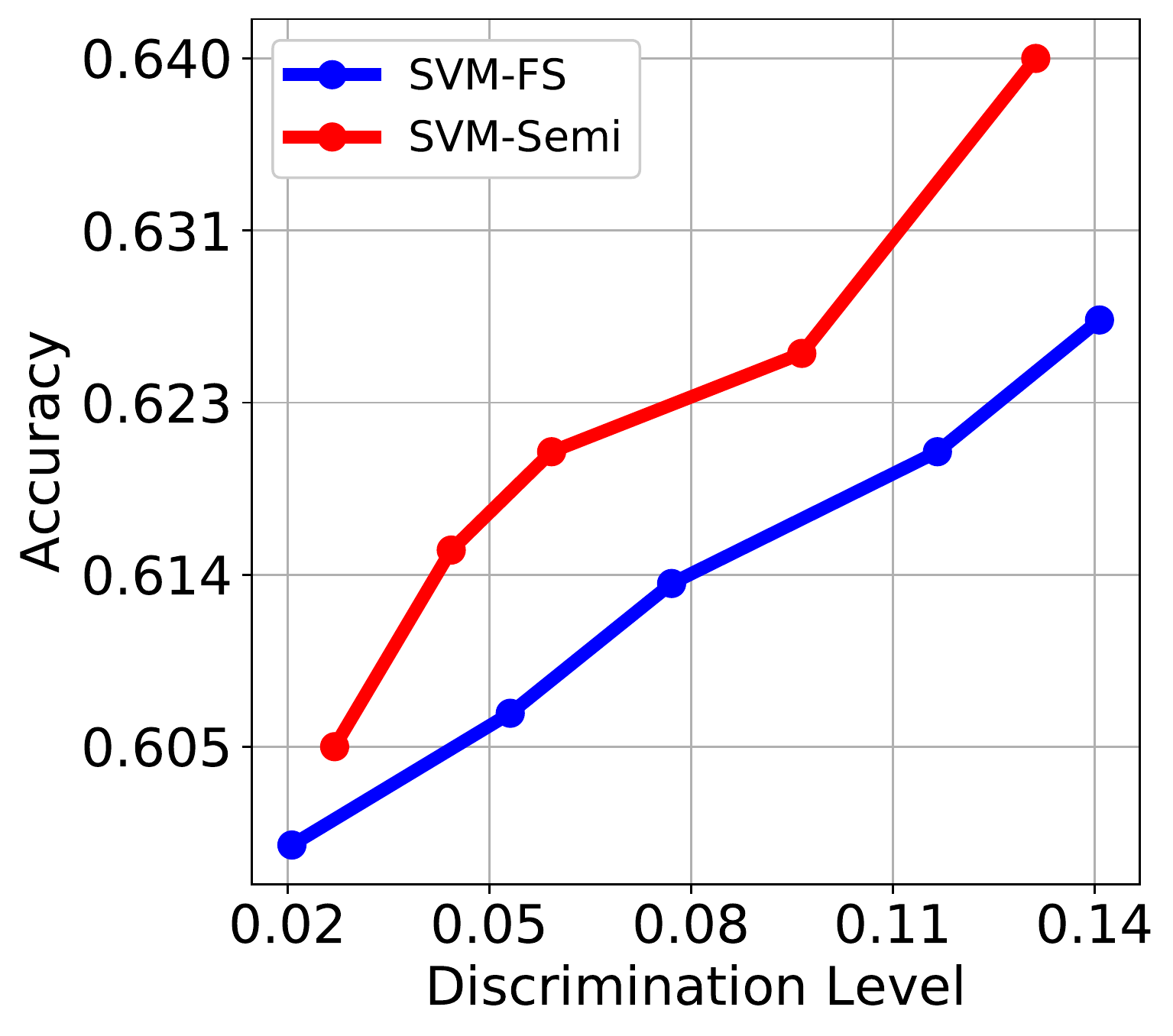}
		\centerline{(d) SVM-Titanic-FNR}
	\end{minipage}
	\caption{The trade-off between accuracy and discrimination in the proposed method Semi (Red), FS (Blue) with LR and SVM in two datasets under the metric of false negative rate. As the threshold of covariance $ c $ increases, accuracy and discrimination increase. }	
\end{figure}

\begin{figure}[ht]
	\begin{minipage}[b]{0.49\linewidth}
		\centering	
		\includegraphics[scale=0.24]{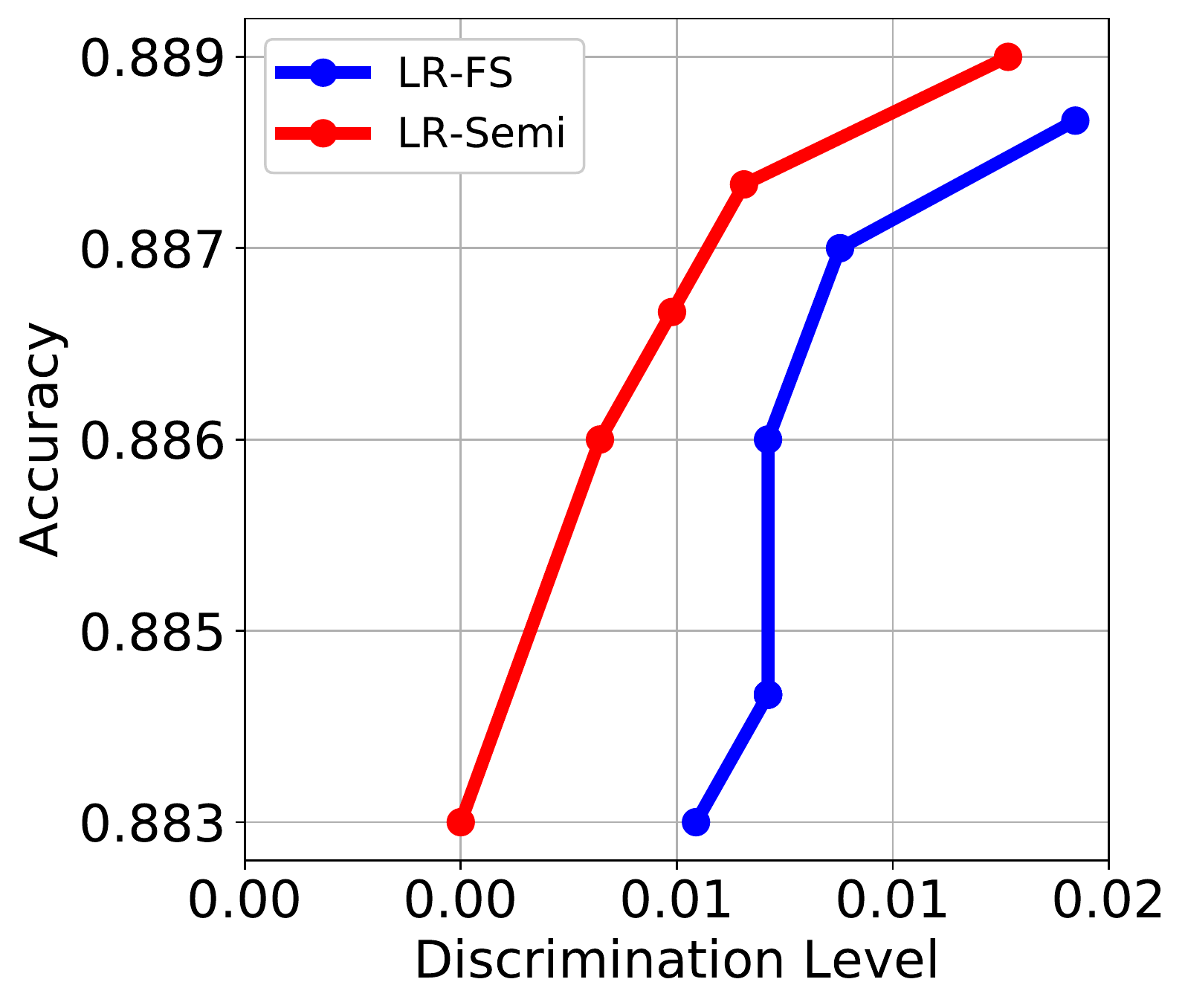}
		\centerline{(a) LR-Bank-FPR}
	\end{minipage}
	\begin{minipage}[b]{0.49\linewidth}
		\centering
		\includegraphics[scale=0.24]{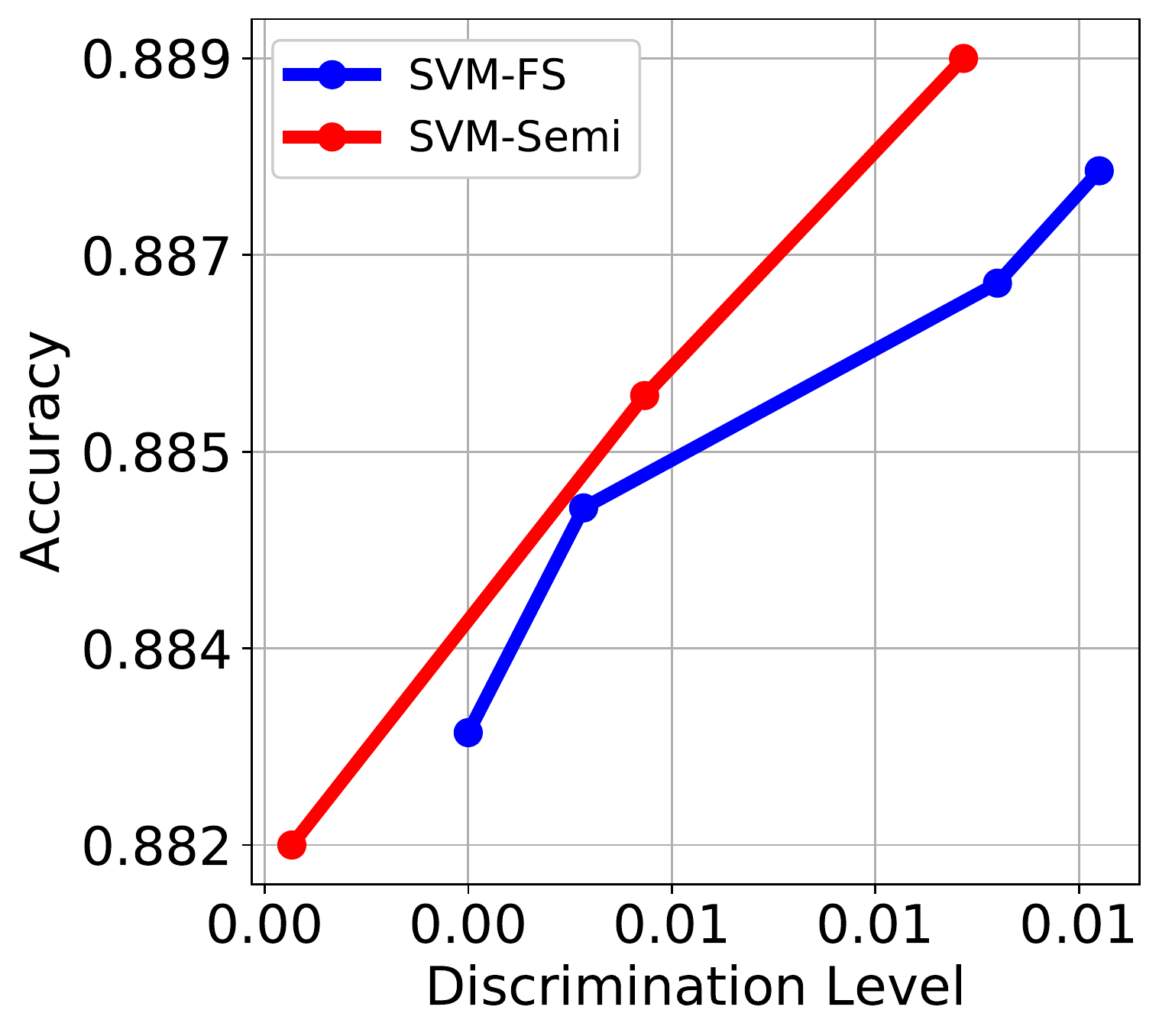}
		\centerline{(b) SVM-Bank-FPR}
	\end{minipage}
	\begin{minipage}[b]{0.49\linewidth}
		\centering
		\includegraphics[scale=0.24]{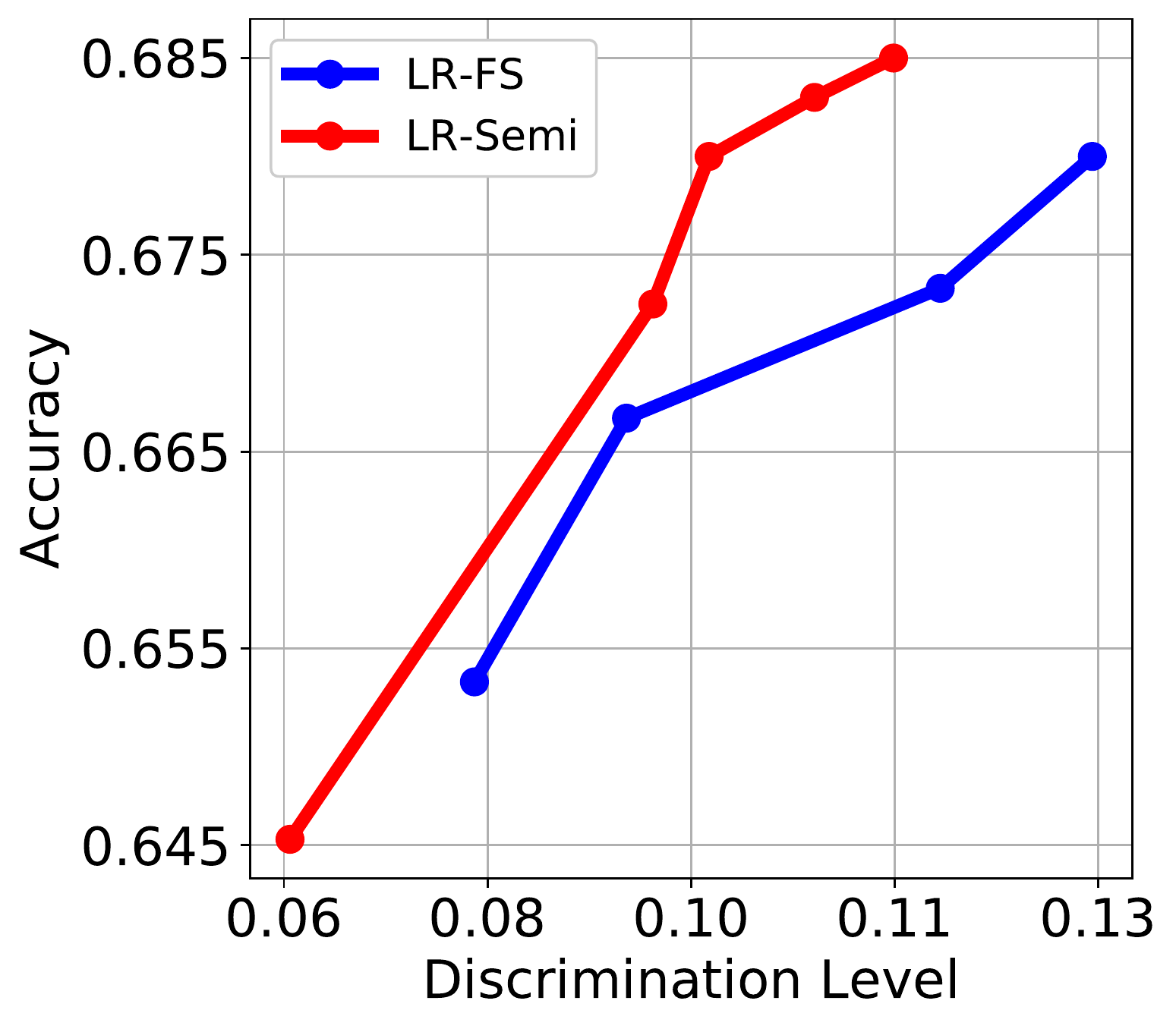}
		\centerline{(c) LR-Titanic-FPR}
	\end{minipage}
	\begin{minipage}[b]{0.49\linewidth}
		\centering
		\includegraphics[scale=0.24]{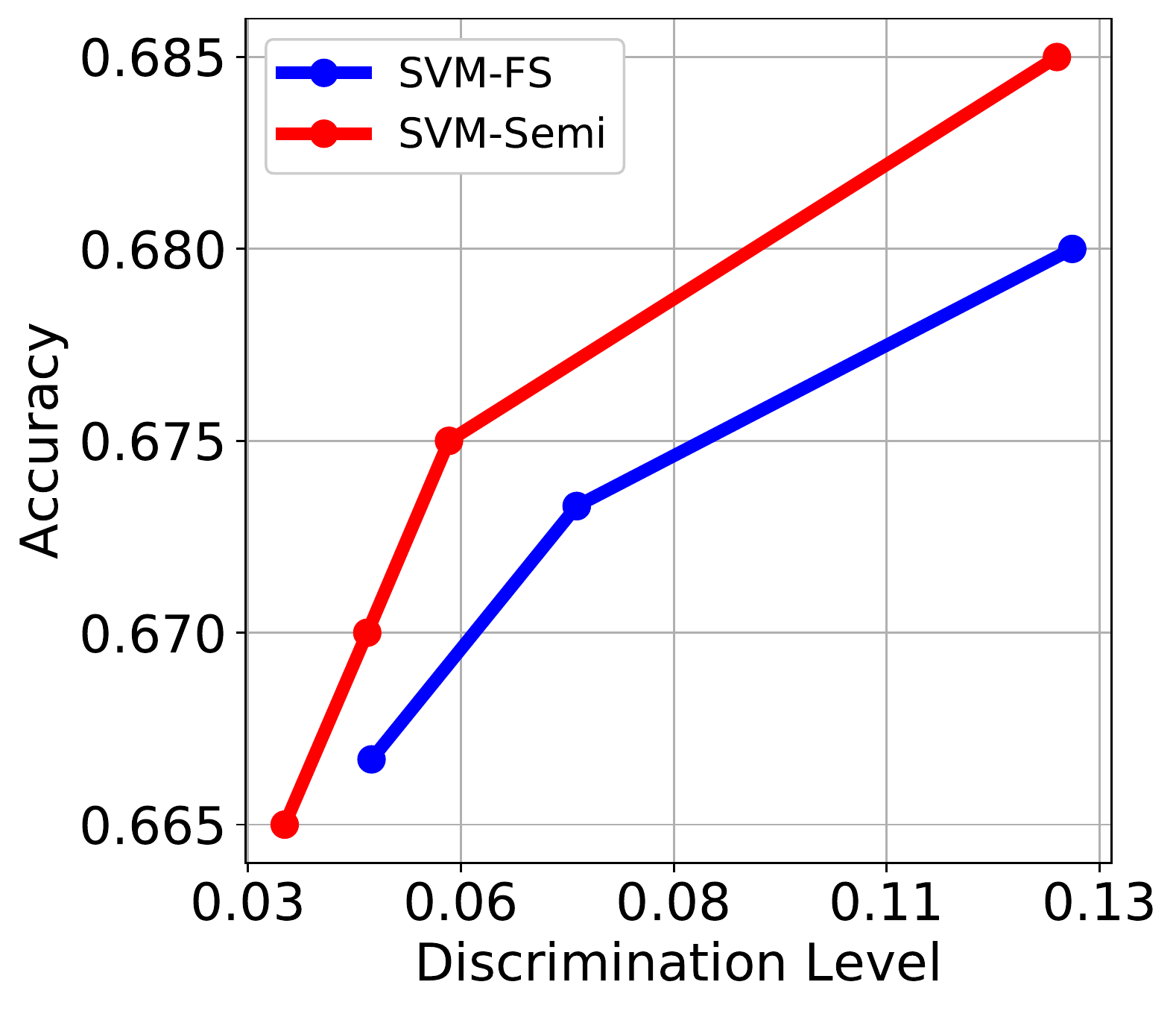}
		\centerline{(d) SVM-Titanic-FPR}
	\end{minipage}
	\caption{The trade-off between accuracy and discrimination in proposed method Semi (Red), FS (Blue) with LR and SVM in two datasets under the metric of false positive rate. As the threshold of covariance $ c $ increases, accuracy and discrimination increase. }	
\end{figure}

\subsubsection{Different Fairness Constraints under OMR}
Table 3 and Table 4 shows that different fairness constraints on labeled and unlabeled data have different impacts on the training results. Due to space limitation, we have only reported the results for the LR under the metric of OMR on the Bank and Titanic datasets.
For these tests, the size of unlabeled data is set to 4,000 data points in the Bank dataset and 400 data points in the Titanic dataset. As shown,  when varying the threshold of covariance $ c $,  different fairness constraints on labeled and unlabeled data have huge difference on the training results. When the fairness constraint is enforced in labeled data, accuracy  and discrimination increases with the increase in $ c $ in the Titanic dataset. This is because a smaller $ c $ enforces the lowest discrimination level, which results in a lower accuracy.

However, when the fairness constraint is enforced in unlabeled data, accuracy and discrimination could decrease with the increase in $ c $. This is because the label of unlabeled data appears in the fairness constraint of disparate mistreatment, and it is updated during the training. This means that the distribution of unlabeled data is not described well during the training. As a result, the fairness constraint on unlabeled data is not that effective.
\begin{table}[]
	\scalebox{0.78}{
		\begin{tabular}{lcccccccc} \hline
			Dataset    & \multicolumn{7}{c}{Bank dataset}                                                                      &             \\ \hline
			Constraint & \multicolumn{2}{l}{Labeled} & \multicolumn{2}{l}{Unlabeled} & \multicolumn{2}{l}{Combined} & \multicolumn{2}{l}{Mixed} \\ \hline
			& Acc          & Dis          & Acc           & Dis           & Acc           & Dis          & Acc         & Dis         \\             \hline
			c=0.0 & 0.8635 & 0.0905 & 0.8407 & 0.1847 & 0.8342 & 0.147  & 0.8605 & 0.0987 \\
c=0.5 & 0.8625 & 0.092  & 0.8402 & 0.1854 & 0.8342 & 0.1442 & 0.8605 & 0.0987 \\
c=1.0 & 0.8638 & 0.0922 & 0.8402 & 0.1854 & 0.835  & 0.1452 & 0.8635 & 0.1071 \\
c=1.5 & 0.8645 & 0.0918 & 0.8407 & 0.1833 & 0.835  & 0.1452 & 0.8635 & 0.1071 \\
c=2.0 & 0.8648 & 0.0907 & 0.841  & 0.1822 & 0.8347 & 0.1462 & 0.8625 & 0.1071 \\
c=2.5 & 0.8652 & 0.0914 & 0.8413 & 0.1812 & 0.8353 & 0.1469 & 0.8635 & 0.1084 \\
c=3.0 & 0.866  & 0.0923 & 0.8413 & 0.1784 & 0.8342 & 0.147  & 0.8627 & 0.1084 \\
c=3.5 & 0.8662 & 0.0927 & 0.8407 & 0.1805 & 0.8342 & 0.147  & 0.8627 & 0.1097 \\
c=4.0 & 0.8665 & 0.093  & 0.841  & 0.1795 & 0.8342 & 0.147  & 0.8627 & 0.1097 \\
c=4.5 & 0.8668 & 0.0919 & 0.8407 & 0.1791 & 0.835  & 0.1452 & 0.8635 & 0.1113 \\
c=5.0 & 0.867  & 0.0909 & 0.8407 & 0.1791 & 0.8355 & 0.1444 & 0.8635 & 0.1113 \\ \hline
		\end{tabular}%
}
\caption{The impact of fairness constraints on different datasets in terms of accuracy (Acc) and discrimination level (Dis) under the fairness metric of overall misclassification rate with LR in the Bank dataset.}
\label{tab:my-table}
\end{table}

\begin{table}[]
	\scalebox{0.78}{
		\begin{tabular}{lcccccccc} \hline
			Dataset    & \multicolumn{7}{c}{Titanic dataset}                                                                      &             \\ \hline
			Constraint & \multicolumn{2}{l}{Labeled} & \multicolumn{2}{l}{Unlabeled} & \multicolumn{2}{l}{Combined} & \multicolumn{2}{l}{Mixed} \\ \hline
			& Acc          & Dis          & Acc           & Dis           & Acc           & Dis          & Acc         & Dis         \\             \hline
			c=0.0 & 0.7448 & 0.0285 & 0.7138 & 0.3996 & 0.7655 & 0.1387 & 0.7483 & 0.0175 \\
			c=0.5 & 0.7483 & 0.0335 & 0.6966 & 0.4386 & 0.7655 & 0.1547 & 0.7483 & 0.0175 \\
			c=1.0 & 0.7517 & 0.0385 & 0.6931 & 0.4656 & 0.7552 & 0.1397 & 0.7517 & 0.0225 \\
			c=1.5 & 0.7552 & 0.0436 & 0.7103 & 0.3946 & 0.7793 & 0.1748 & 0.7448 & 0.0445 \\
			c=2.0 & 0.7552 & 0.0436 & 0.7069 & 0.4216 & 0.7724 & 0.1648 & 0.7483 & 0.0495 \\
			c=2.5 & 0.7586 & 0.0326 & 0.7103 & 0.4106 & 0.7759 & 0.1378 & 0.7448 & 0.0605 \\
			c=3.0 & 0.7552 & 0.0596 & 0.7552 & 0.2678 & 0.7552 & 0.0596 & 0.7483 & 0.0655 \\
			c=3.5 & 0.7552 & 0.0596 & 0.6931 & 0.4656 & 0.7552 & 0.0596 & 0.7483 & 0.0816 \\
			c=4.0 & 0.7586 & 0.0646 & 0.7103 & 0.4106 & 0.7586 & 0.0646 & 0.7517 & 0.0866 \\
			c=4.5 & 0.7586 & 0.0646 & 0.7138 & 0.3996 & 0.7586 & 0.0646 & 0.7517 & 0.0866 \\
			c=5.0 & 0.7552 & 0.0756 & 0.7103 & 0.4106 & 0.7552 & 0.0756 & 0.7483 & 0.0816 \\ \hline
		\end{tabular}%
}
\caption{The impact of fairness constraints on different datasets in terms of accuracy (Acc) and discrimination level (Dis) under the fairness metric of overall misclassification rate with LR in the Titanic dataset.}
\label{tab:my-table}
\end{table}
\subsubsection{The Impact of Unlabeled Data under OMR}
\begin{figure}[ht]
	\begin{minipage}[b]{0.49\linewidth}
		\centering	
		\includegraphics[scale=0.22]{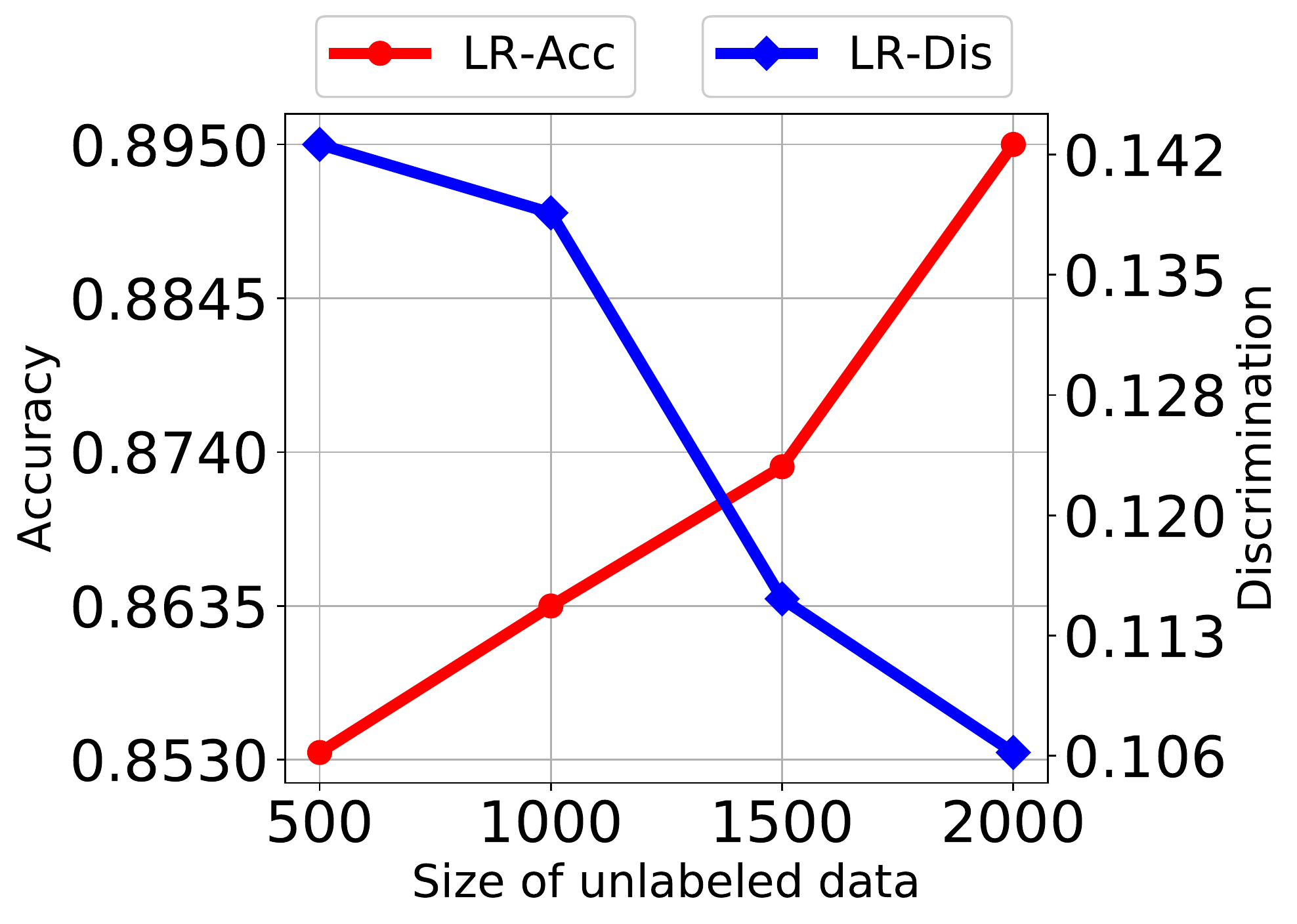}
		\centerline{(a) LR-Bank-OMR}
	\end{minipage}
	\begin{minipage}[b]{0.49\linewidth}
		\centering
		\includegraphics[scale=0.22]{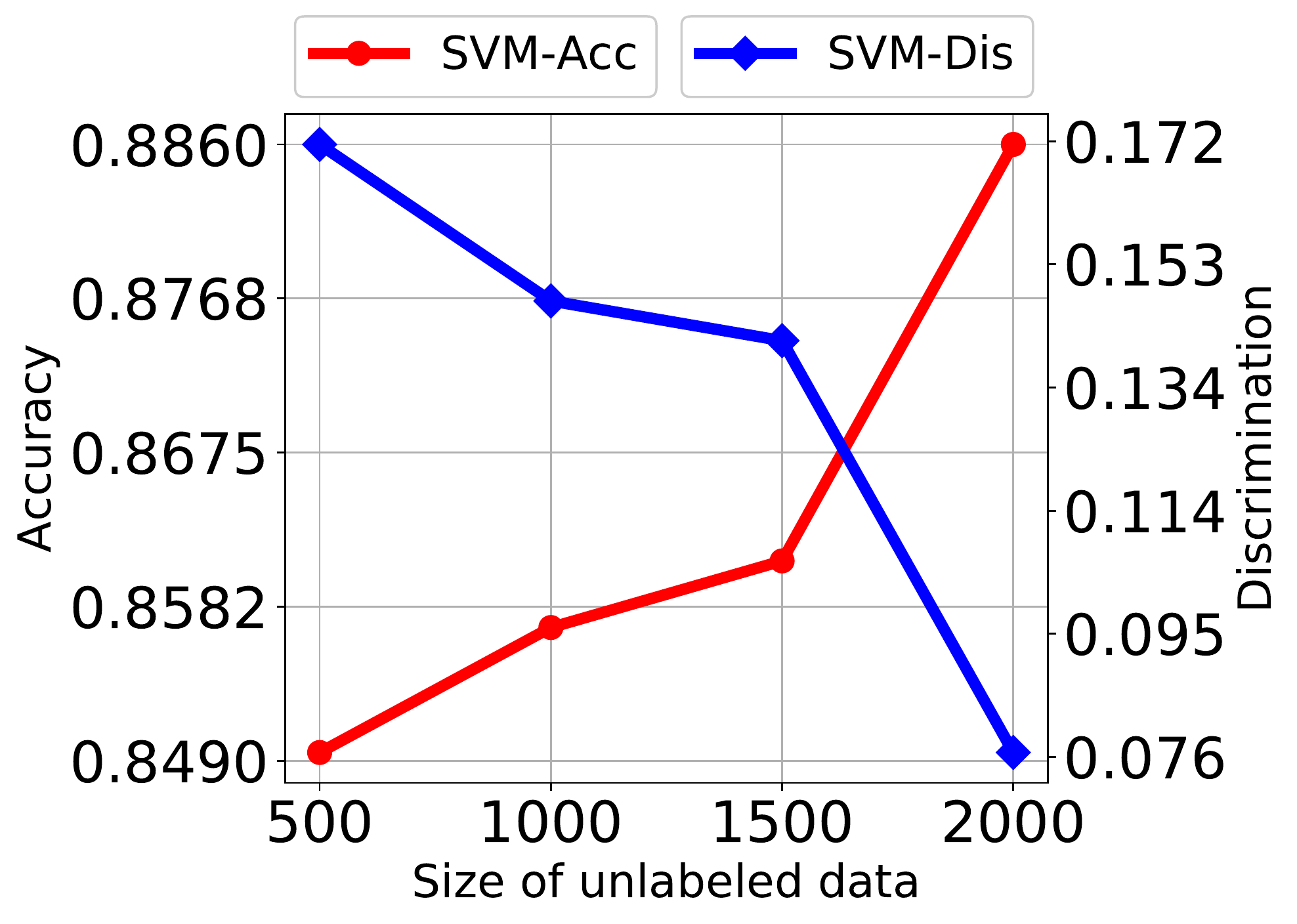}
		\centerline{(b) SVM-Bank-OMR}
	\end{minipage}
	\begin{minipage}[b]{0.49\linewidth}
		\centering
		\includegraphics[scale=0.22]{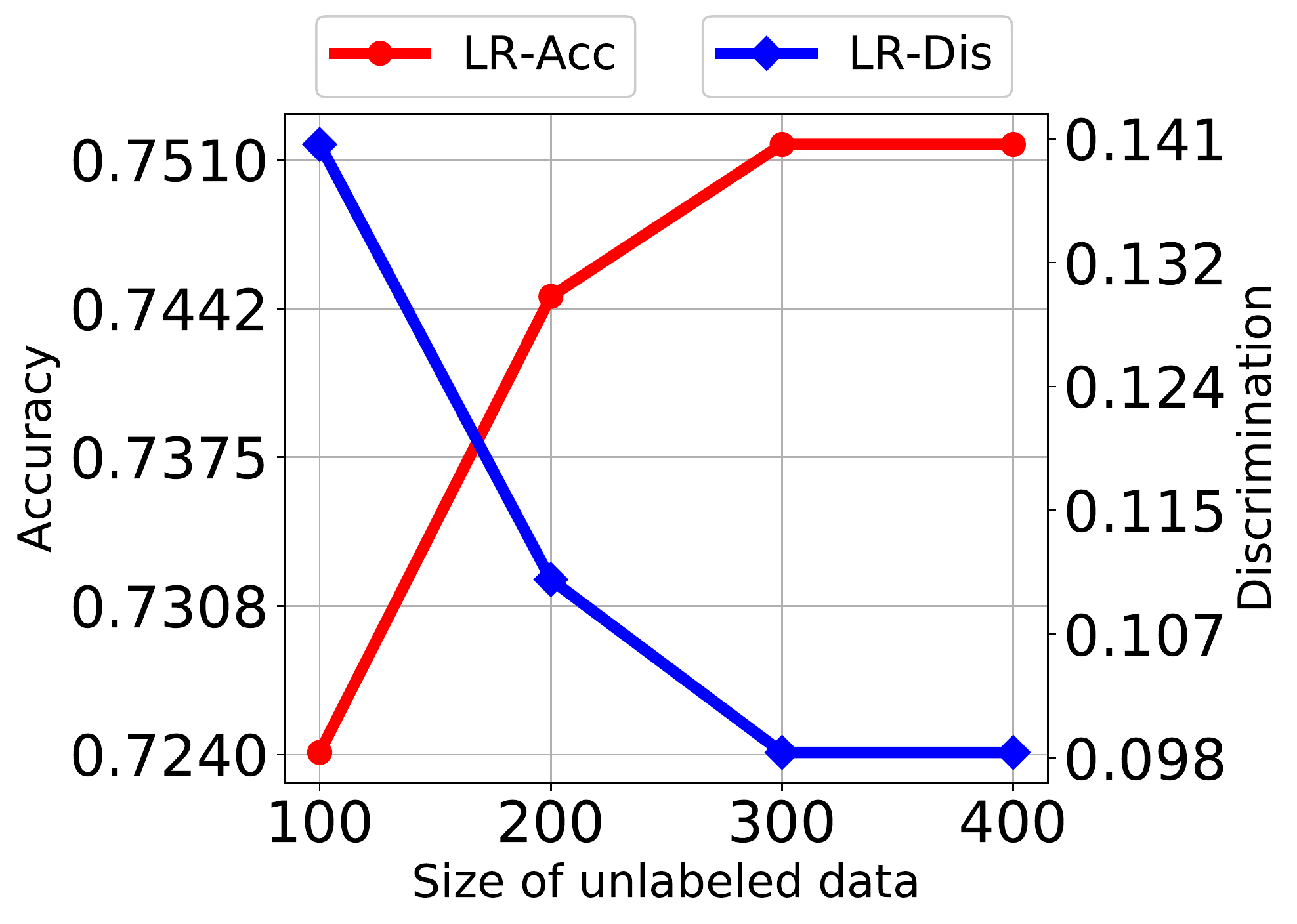}
		\centerline{(c) LR-Titanic-OMR}
	\end{minipage}
	\begin{minipage}[b]{0.49\linewidth}
		\centering
		\includegraphics[scale=0.22]{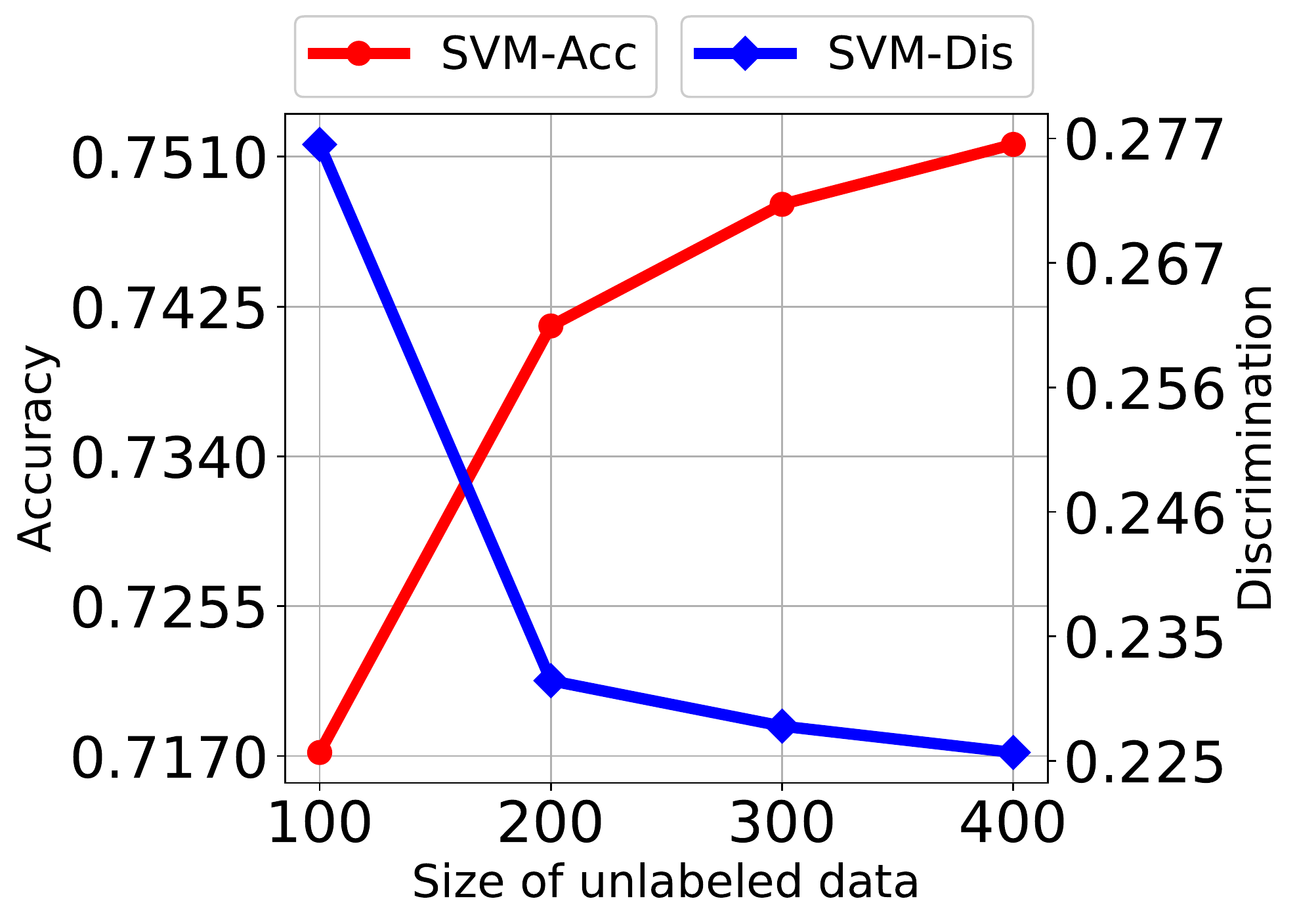}
		\centerline{(d) SVM-Titanic-OMR}
	\end{minipage}
	\caption{The impact of the amount of unlabeled data in the training set on accuracy (Red) and discrimination level (Blue) under the fairness metric of overall mistreatment rate with LR and SVM in two datasets. The X-axis is the size of unlabeled dataset; left y-axis is accuracy; and right y-axis is discrimination level. }	
\end{figure}
For these experiments, we show the impact of unlabeled data on OMR. The covariance threshold is set as $ c=1 $ for the Bank and Titanic datasets.
Figure 6 shows accuracy and discrimination level varies given different size of unlabeled data with LR and SVM on two datasets. 
As shown, before the peak is reached, as the amount of unlabeled data increases in the two data sets,  accuracy will also increase.
Discrimination level decreases at the beginning, and then stabilize in the Titanic dataset. These results indicate that discrimination in variance decreases as the amount of unlabeled data in the training set increases.
\subsection{Discussion and Summary}

\subsubsection{Discussion}
We discuss on how the proposed framework is able to reduce discrimination in terms of discrimination decomposition into discrimination in bias, variance and noise. 
Discrimination in bias depends on the model choice. Discrimination in variance relates to the size of training data. Our framework uses unlabeled data to expand the size of training data, and thus reduce the discrimination in variance. 
Discrimination in noise depends on the quality of training data. In our framework, discrimination in noise also depends on the label propagation. Predicting labels for  unlabeled data may bring discrimination in noise. This discrimination can be adjusted by the threshold $ c $ in the fairness constraint. A smaller threshold generates a smaller discrimination in noise. The reduction in discrimination in variance is generally more than the discrimination in noise induced by label propagation. Thus, when the same model is used, unlabeled data helps to lesson discrimination.
\subsubsection{Summary}
From these experiments, we can obtain some conclusions.
1) The proposed framework can make use of unlabeled data to achieve a better trade-off between accuracy and discrimination.
2) Under the fairness metric of disparate impact,  the fairness constraint on mixed labeled and unlabeled data generally has the best trade-off between accuracy and discrimination. Under the fairness metric of disparate mistreatment,  the fairness constraint on labeled data is used to achieve the trade-off between accuracy and discrimination.
3) More unlabeled data generally helps to make a better compromise between accuracy and discrimination.
4) Model choice can affect the trade-off between accuracy and discrimination. Our experiments show that SVM is more friendly to achieve a better trade-off than LR.

\section{Related Work}
In recent years, a large number of work have studied the fairness in machine learning, and we classify them into two streams.
\subsection{Fair Supervised Learning}
Methods for fair supervised learning include pre-processing, in-processing and post-processing methods.
In pre-processing,  discrimination is eliminated by guiding the distribution of training data towards to a fairer direction \cite{Kamiran2012} or by transforming the training data into a new space \cite{Zemel2013,Song2019,calmon2017optimized,Feng2019,zhao2019inherent}. The main advantage of the pre-processing method is that it does not require changes to the machine learning algorithm, so it is very simple to use.
In in-processing, discrimination is constrained by fair constraints or a regularizer during the training phase. For example, Zafar et al. \cite{Kamishima2012} used regularizer term to penalize discrimination in the learning objective.  
\cite{Zafar2017,NIPS2016_6316,Zafar2017a} designed the convex fairness constraint, called decision boundary covariance to achieve fair classification for classifiers. Recent work presented the constrained optimization problem as a two-player game, and formalized the definition of fairness as a linear inequality \cite{donini2018empirical,agarwal2018,cotter2019optimization}. This category is more flexible for optimizing different fair constraints, and solutions using this method are considered to be the most robust. In addition, these methods have shown good results in terms of accuracy and fairness.
A third approach to achieving fairness is post-processing, where a learned classifier is modified to adjust the decisions to be non-discriminatory for different groups
\cite{Hardt2016,kim2019multiaccuracy,lohia2019bias}. Post-processing does not need changes in the classifier, but it cannot guarantee a optimal classifier.

\subsection{Fair Unsupervised Learning}
Chierichetti et al. \cite{chierichetti2017fair} was the first to study fairness in clustering problems. Their solution, under both k-center and k-median objectives, was required every group to be (approximately) equally represented in each cluster. Many subsequent works have since been undertaken on the subject of fair clustering. Among these,
Rosner et al. \cite{rsner_et_al:LIPIcs:2018:9100} extended fair clustering to more than two groups.
Chen et al. \cite{DBLP:journals/corr/abs-1812-10854} consider the fair k-means problem in the streaming model, define fair coresets and show how to compute them in a streaming setting, resulting in significant reduction in the input size. 
Bera et al. \cite{NIPS2019_8741} presented a more generalized approach to fair clustering, providing a tunable notion of fairness in clustering.
Kleindessner et al. \cite{pmlr-v97-kleindessner19a} studied a version of constrained spectral clustering incorporating the fairness constraints. 
\subsection{Comparing with other work}
Existing fair learning methods mainly focus on supervised and unsupervised learning, and cannot be directly applied to SSL.
As far as we know, only \cite{NIPS2019_9437,noroozi2019leveraging,zhang2020fairness} has explored fairness in SSL. Chzhen et al. \cite{NIPS2019_9437} studied Bayes classifier under the fairness metric of equal opportunity, where labeled data is used to learn the output conditional probability, and unlabeled data is used for to calibrate threshold in the post-processing phase. However, unlabeled data is not fully used to eliminate discrimination, and the proposed method only applies in equal opportunity. 
In \cite{noroozi2019leveraging}, the proposed method is built on neural networks for SSL in the in-processing phase, where unlabeled data is marked labels with pseudo labeling.  Zhang et al. \cite{zhang2020fairness} proposed a pre-processing framework which includes pseudo labeling, re-sampling and ensemble learning to remove representation discrimination.
Our solution will focus on margin-based classifier in the in-processing stage, as in-processing methods have demonstrated good flexibility in both balancing fairness and supporting multiple classifiers and fairness metrics.
\section{Conclusion}
In this paper, we propose a framework of fair semi-supervised learning that operates during in-processing phase. 
Our framework is formulated as an optimization problem with the goal of finding weights and labeling unlabeled data by minimizing the loss function subject to fairness constraints. We analyze several different cases of fairness constraints for their effects on the optimization problem plus the accuracy and discrimination level in the results. 
A theoretical analysis on three sources of discrimination – bias, variance and noise decomposition – shows that unlabeled data  is a viable option for achieving a better trade-off between accuracy and fairness. Our experiments confirm this analysis, showing that the proposed framework provides accuracy and fairness at high levels in semi-supervised settings.

A further research direction is to explore ways of achieving fair SSL under other fairness metrics, such as individual fairness.
Further, in this paper, we assume that labeled data and unlabeled have a similar data distribution. However, this assumption may not be hold in some practical situations. Therefore, another research direction is how to achieve fair SSL when labeled and unlabeled data distribution is different.

\appendices

\section*{Acknowledgment}
This work is supported by an ARC Discovery Project (DP190100981) from Australian Research Council, Australia; and in part by NSF under grants III-1526499, III-1763325, III-1909323, and CNS-1930941. 
\ifCLASSOPTIONcaptionsoff
  \newpage
\fi

\bibliographystyle{IEEEtran}
\bibliography{ref}
\begin{IEEEbiography}[{\includegraphics[width=1in,height=1.25in,clip,keepaspectratio]{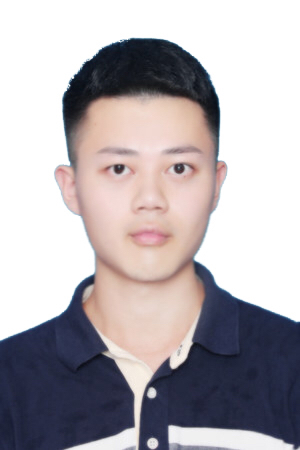}}]{Tao Zhang}
 he works towards his Ph.D degree with the school of Computer Science in the University of Technology Sydney, Australia.
	His research interests include privacy preserving, algorithmic fairness, and machine learning.
\end{IEEEbiography}
\begin{IEEEbiography} [{\includegraphics[width=1in,height=1.25in,clip,keepaspectratio]{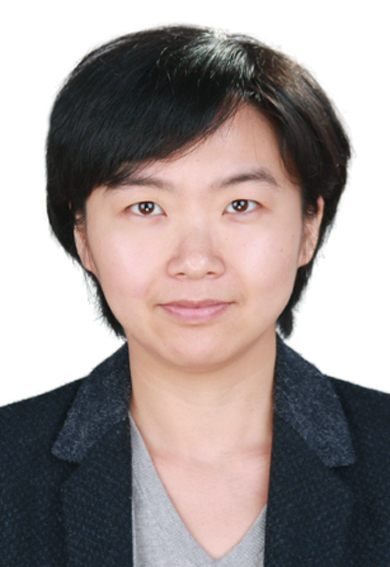}}]{Tianqing Zhu}
	 received the B.Eng. degree in chemistry and M.Eng. degree in automation from Wuhan University, Wuhan, China, in 2000 and 2004, respectively, and the Ph.D. degree in computer science from Deakin University, Geelong, Australia, in 2014. 
	 
	 She is currently an Associate Professor in the Faulty of Engineering and Information Technology with the School of Computer Science, University of Technology Sydney, Sydney, Australia. Before that, she was a Lecturer in the School of Information Technology, Deakin University, from 2014 to 2018. Her research interests include privacy preserving, data mining, and network security
\end{IEEEbiography}
\begin{IEEEbiography}[{\includegraphics[width=1in,height=1.25in,clip,keepaspectratio]{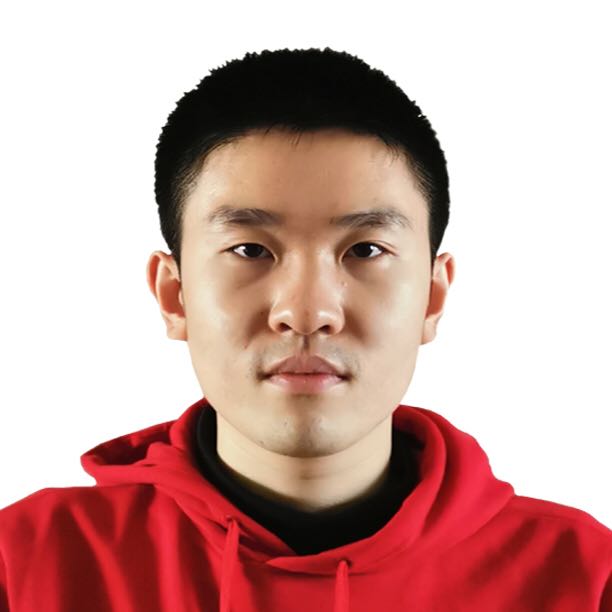}}]{Mengde Han}
	is a PhD student at University of Technology Sydney with a focus on Local Differential Privacy. He completed his Master's at the Johns Hopkins University.
\end{IEEEbiography}
\begin{IEEEbiography}[{\includegraphics[width=1in,height=1.25in,clip,keepaspectratio]{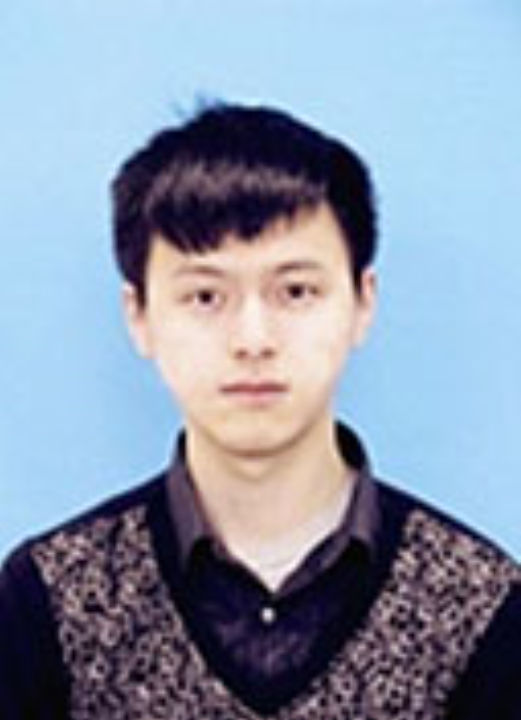}}]{Jing Li} 
	 received the B.Eng and M.Eng degrees in computer science and technology from Northwestern Polytechnical University, Xi’an, China, in 2015 and 2018, respectively.  
	 
	 Currently, he is pursuing the Ph.D degree with the Centre for Artificial Intelligence in the University of Technology Sydney, Australia. His research interests include machine learning and privacy preserving.	
\end{IEEEbiography}
\begin{IEEEbiography}[{\includegraphics[width=1in,height=1.25in,clip,keepaspectratio]{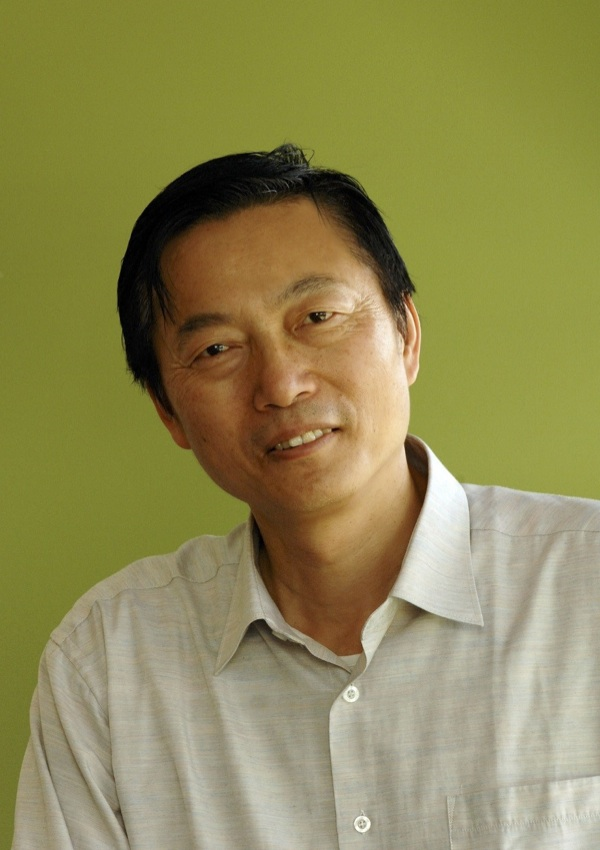}}]
	{Wanlei Zhou}
	received the B.Eng and M.Eng degrees from Harbin Institute of Technology, Harbin, China in 1982 and 1984, respectively, and the PhD degree from The Australian National University, Canberra, Australia, in 1991, all in Computer Science and Engineering. He also received a DSc degree (a higher Doctorate degree) from Deakin University in 2002. He is currently the Head of School of Computer Science in University of Technology Sydney (UTS). Before joining UTS, Professor Zhou held the positions of Alfred Deakin Professor, Chair of Information Technology, and Associate Dean (International Research Engagement) of Faculty of Science, Engineering and Built Environment, Deakin University. His research interests include security and privacy, bioinformatics, and e-learning. Professor Zhou has published more than 400 papers in refereed international journals and refereed international conferences proceedings, including many articles in IEEE transactions and journals.
\end{IEEEbiography}
\begin{IEEEbiography}[{\includegraphics[width=1in,height=1.25in,clip,keepaspectratio]{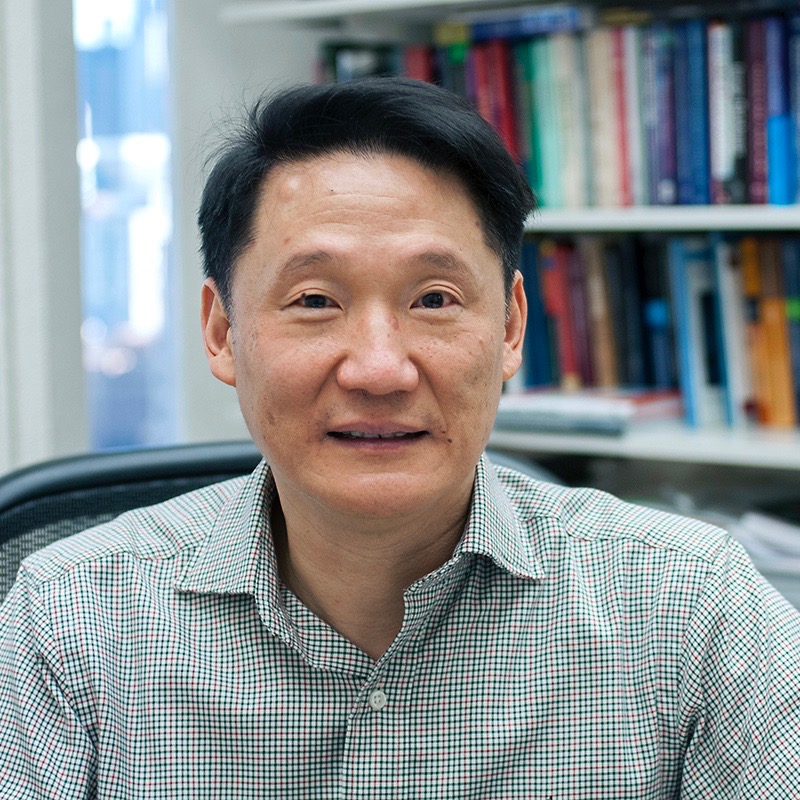}}]
	{Philip S. Yu}
	received the B.S. degree in electrical engineering from National Taiwan University, Taipei, Taiwan, the M.S. and Ph.D. degrees in electrical engineering from Stanford University, Stanford, CA, USA, and the M.B.A. degree from New York University, New York, NY, USA. He was with IBM, Armonk, NY, USA, where he was a Manager of the Software Tools and Techniques Department with the Thomas J. Watson Research Center. He is a Distinguished Professor of computer science with the University of Illinois at Chicago, Chicago, IL, USA, where he also holds the Wexler Chair in information technology. He has published over 1200 papers in peer-reviewed journals, such as the IEEE Transactions on Knowledge and Data Engineering, ACM Transactions on Knowledge Discovery from Data, VLDBJ, and the ACM Transactions on Intelligent Systems and Technology and conferences, such as KDD, ICDE, WWW, AAAI, SIGIR, ICML, and CIKM. He holds or has applied for over 300 U.S. patents. His current research interests include data mining, data streams, databases, and privacy. Dr. Yu was a recipient of the ACM SIGKDD 2016 Innovation Award for his influential research and scientific contributions on mining, fusion, and anonymization of Big Data and the IEEE Computer Society 2013 Technical Achievement Award. He was the Editor-in-Chief of ACM Transactions on Knowledge Discovery from Data (2011-2017) and IEEE Transactions on Knowledge and Data Engineering (2001-2004). He is a fellow of ACM.
\end{IEEEbiography}
\end{document}